\documentclass[11pt]{article}

\usepackage{amsmath,amssymb,amsfonts,amsthm}
\usepackage{fullpage}
\usepackage{booktabs}

\newtheorem{definition}{Definition}
\newtheorem{theorem}{Theorem}

\newtheorem{lemma}{Lemma}
\newtheorem{example}{Example}
\newtheorem{remark}{Remark}

\newcommand{\game}{G}
\newcommand{\states}{S}
\newcommand{\act}{\mathcal{A}}
\newcommand{\trans}{\delta}
\newcommand{\obs}{\mathcal{O}}
\newcommand{\obsmap}{\gamma}
\newcommand{\distr}{\mathcal{D}}
\newcommand{\powset}{\mathcal{P}}
\newcommand{\reward}{\mathsf{r}}
\newcommand{\limavg}{\mathsf{LimAvg}}
\newcommand{\limavgone}{\limavg_{=1}}
\newcommand{\limavghalf}{\limavg_{>\frac{1}{2}}}

\newcommand{\almostm}{\mathsf{Almost}_{\mathcal{M}}}
\newcommand{\set}[1]{\{#1\}}
\newcommand{\straa}{\sigma}
\newcommand{\restr}{\!\!\upharpoonright}
\newcommand{\rec}{\mathsf{RecFun}}
\newcommand{\win}{\mathsf{AWFun}}
\newcommand{\actions}{\mathsf{Act}}
\newcommand{\supp}{\mathrm{Supp}}
\newcommand{\prb}{\mathbb{P}}
\newcommand{\good}{\mathsf{good}}
\newcommand{\bad}{\mathsf{bad}}
\newcommand{\mem}{\mathsf{Mem}^{*}}

\newcommand{\last}{\mathsf{Last}}
\newcommand{\belief}{\mathcal{B}}
\newcommand{\wh}{\widehat}
\newcommand{\av}{\Gamma}
\newcommand{\allow}{{\textsf{Allow}}}
\newcommand{\safe}{{\mathsf{Safe}}}
\newcommand{\reach}{{\mathsf{Reach}}}
\newcommand{\Cone}{\mathsf{Cone}}
\newcommand{\pre}{{\textsf{Pre}}}
\newcommand{\apre}{{\textsf{Apre}}}
\newcommand{\obscover}{{\textsf{ObsCover}}}

\newcommand{\Reach}{\mathsf{Reach}}
\newcommand{\obsset}{O}
\newcommand{\ov}{\overline}
\newcommand{\Rec}{\mathsf{Rec}}
\newcommand{\target}{\ov{T}}
\newcommand{\targetsafe}{\ov{F}}

\newcommand{\wb}{\overline}

\def \pseudo {collapsed}

\def\projected {collapsed }
\def\projection {collapsed }
\newcommand{\prst}{\mathsf{CoSt}}
\newcommand{\prgr}{\mathsf{CoGr}}

\newcommand{\cale}{\mathcal{E}}
\newcommand{\Exp}{\mathbb{E}}
\newcommand{\red}{\mathsf{Red}}

\newcommand{\automaton}{\mathsf{P}}
\newcommand{\wt}{\widetilde}
\newcommand{\w}{\mathsf{AW}}

\usepackage{tikz}
\usetikzlibrary{fit}
\usetikzlibrary{backgrounds}
\usetikzlibrary{automata}
\usetikzlibrary{shapes}
\usetikzlibrary{matrix}
\usetikzlibrary{fit}
\usetikzlibrary{calc}

\tikzstyle{Player1}=[circle, thick, minimum size=0.6cm, inner sep=0cm,draw=black]
\tikzstyle{State}=[circle, thick, minimum size=0.6cm, inner sep=0cm,draw=black]
\tikzstyle{Final}=[circle, accepting, thick, minimum size=0.6cm, inner sep=0cm,draw=black]
\tikzstyle{RState}=[circle, very thick, minimum size=0.8cm, inner sep=0cm,draw=red]

\title{POMDPs under Probabilistic Semantics \\(Full Version)}
\author{Krishnendu Chatterjee (IST Austria) \and Martin Chmel\'ik (IST Austria)}
\date{}
\begin{document}

\maketitle

\begin{abstract}
We consider partially observable Markov decision processes (POMDPs)
with limit-average payoff, where a reward value in the interval $[0,1]$ is 
associated to every transition, and the payoff of an infinite path is the 
long-run average of the rewards.
We consider two types of path constraints: 
(i)~quantitative constraint defines the set of paths where the payoff is 
at least a given threshold $\lambda_1 \in (0,1]$; and 
(ii)~qualitative constraint which is a special case of quantitative constraint 
with $\lambda_1=1$.
We consider the computation of the almost-sure winning set, where the 
controller needs to ensure that the path constraint is satisfied with 
probability~1. 
Our main results for qualitative path constraint are as follows:
(i)~the problem of deciding the existence of a finite-memory controller is 
EXPTIME-complete; and 
(ii)~the problem of deciding the existence of an infinite-memory controller is 
undecidable.
For quantitative path constraint we show that the problem of deciding the 
existence of a finite-memory controller is undecidable.
\end{abstract}

\section{Introduction}

\noindent{\bf Partially observable Markov decision processes (POMDPs).}
\emph{Markov decision processes (MDPs)} are standard models for 
probabilistic systems that exhibit both probabilistic 
and nondeterministic behavior~\cite{Howard}.
MDPs have been used to model and solve control problems for stochastic 
systems~\cite{FV97,Puterman}: 
nondeterminism represents the freedom of the controller to choose a 
control action, while the probabilistic component of the behavior describes the 
system response to control actions. 
In \emph{perfect-observation (or perfect-information) MDPs (PIMDPs)} the 
controller can observe the current state of the system to choose the next 
control actions, whereas in \emph{partially observable MDPs (POMDPs)} the state space is 
partitioned according to observations that the controller can observe, 
i.e., given the current state, the controller can only view the observation 
of the state (the partition the state belongs to), but not the precise 
state~\cite{PT87}.
POMDPs provide the appropriate model to study a wide variety of applications 
such as in computational biology~\cite{Bio-Book}, 
speech processing~\cite{Mohri97}, image processing~\cite{IM-Book}, 
robot planning~\cite{KGFP09,kaelbling1998planning}, reinforcement learning~\cite{LearningSurvey}, 
to name a few. 
POMDPs also subsume many other powerful computational models such as 
probabilistic finite automata (PFA)~\cite{Rabin63,PazBook} 
(since probabilistic finite automata (aka blind POMDPs) are a special case of 
POMDPs with a single observation).

\smallskip\noindent{\bf Limit-average payoff.} 
A \emph{payoff} function maps every infinite path (infinite sequence of 
state action pairs) of a POMDP to a real value.
The most well-studied payoff in the setting of POMDPs is the 
\emph{limit-average} payoff where every state action pair is 
assigned a real-valued reward in the interval $[0,1]$ and 
the payoff of an infinite path is the long-run average of the 
rewards on the path~\cite{FV97,Puterman}.
POMDPs with limit-average payoff provide the theoretical
framework to study many important problems of practical relevance,
including probabilistic planning and several stochastic optimization problems~\cite{kaelbling1998planning,cassandra1995acting,MHC03,Meuleau:1999:LFC:2073796.2073845,williams2007partially}.

\smallskip\noindent{\bf Expectation vs probabilistic semantics.}
Traditionally, MDPs with limit-average payoff have been studied with the 
\emph{expectation} semantics, where the goal of the controller is to maximize 
the expected limit-average payoff.
The expected payoff value can be $\frac{1}{2}$ when with 
probability~$\frac{1}{2}$ the payoff is~1, and with remaining probability the 
payoff is~0. 
In many applications of system analysis (such as robot planning and 
control) the relevant question is the probability measure of the paths that 
satisfy certain criteria, e.g., whether the probability measure of the paths 
such that the limit-average payoff is~1 (or the payoff is at least 
$\frac{1}{2}$) is at least a given threshold (e.g., see~\cite{andersson2006symbolic,KGFP09}).
We classify the path constraints for limit-average payoff as follows:
(1)~\emph{quantitative constraint} that defines the set of paths with 
limit-average payoff at least $\lambda_1$, for a threshold $\lambda_1 \in 
(0,1]$; and (2)~\emph{qualitative constraint} is the special case of 
quantitative constraint that defines the set of paths with limit-average 
payoff~1 (i.e., the special case with $\lambda_1=1$).
We refer to the problem where the controller must satisfy a path 
constraint with a probability threshold $\lambda_2 \in (0,1]$ as the 
\emph{probabilistic} semantics.
An important special case of probabilistic semantics is the \emph{almost-sure}
semantics, where the probability threshold is~1.
The almost-sure semantics is of great importance because there are many 
applications where the requirement is to know whether the correct behavior 
arises with probability~1.
For instance, when analyzing a randomized embedded scheduler, the relevant 
question is whether every thread progresses with probability~1.
Even in settings where it suffices to satisfy certain specifications with 
probability $\lambda_2<1$, the correct choice of $\lambda_2$ is a challenging 
problem, due to the simplifications introduced during modeling.
For example, in the analysis of randomized distributed algorithms it is 
quite common to require correctness with probability~1 
(e.g.,~\cite{PSL00,Sto02b}). 
Besides its importance in practical applications, almost-sure convergence,  
like convergence in expectation, is a fundamental concept in probability theory,
and provide stronger convergence guarantee than convergence in expectation~\cite{Durrett}.

\smallskip\noindent{\bf Previous results.}
There are several deep undecidability results established for the 
special case of probabilistic finite automata (PFA) (that immediately imply 
undecidability for the more general case of POMDPs).
The basic undecidability results are for PFA over finite words: 
The emptiness problem for PFA under probabilistic semantics is 
undecidable over finite words~\cite{Rabin63,PazBook,CL89}; and it was shown 
in~\cite{MHC03} that even the following approximation version is undecidable: 
for any fixed $0<\epsilon <\frac{1}{2}$, given a probabilistic finite automaton 
and the guarantee that either (a)~all words are accepted with probability at 
least $1-\epsilon$; or 
(ii)~all words are accepted with probability at most $\epsilon$;
decide whether it is case~(i) or case~(ii).
The almost-sure problem for probabilistic automata over finite words reduces to 
the non-emptiness question of universal automata over finite words and 
is PSPACE-complete.
However, another related decision question whether for every $\epsilon>0$
there is a word that is accepted with probability at least 
$1-\epsilon$ (called the value~1 problem) is undecidable for probabilistic 
automata over finite words~\cite{GO10}.
Also observe that all undecidability results for probabilistic automata over
finite words carry over to POMDPs where the controller is restricted to 
finite-memory strategies.

\smallskip\noindent{\bf Our contributions.}
Since under the general probabilistic semantics, the decision problems are
undecidable even for PFA, we consider POMDPs with limit-average payoff 
under the almost-sure semantics. 
We present a complete picture of decidability as well as optimal complexity.

\begin{enumerate}
\item \emph{(Almost-sure winning for qualitative constraint).}
We first consider limit-average payoff with qualitative constraint 
under almost-sure semantics.
We show that \emph{belief-based} strategies are not sufficient 
(where a belief-based strategy is based on the subset construction that 
remembers the possible set of current states): we show that there exist POMDPs 
with limit-average payoff with qualitative constraint where finite-memory 
almost-sure winning strategy exists but there exists no belief-based 
almost-sure winning strategy. 
Our counter-example shows that standard techniques based on subset construction 
(to construct an exponential size PIMDP) are not adequate to solve the 
problem.   
We then show one of our main result that given a POMDP with $|S|$ states and 
$|\act|$ actions, if there is a finite-memory almost-sure winning strategy to 
satisfy the limit-average payoff with qualitative constraint, then there is an 
almost-sure winning strategy that uses at most $2^{3\cdot |S| + |\act|}$ memory.
Our exponential memory upper bound is asymptotically optimal, as even for PFA 
over finite words, exponential memory is required for almost-sure winning
(follows from the fact that the shortest witness word for non-emptiness 
of universal finite automata is at least exponential).
We then show that the problem of deciding the existence of finite-memory 
almost-sure winning strategy for limit-average payoff with 
qualitative constraint is EXPTIME-complete for POMDPs.
In contrast to our result for finite-memory strategies, we establish that the 
problem of deciding the existence of infinite-memory almost-sure winning 
strategy for limit-average payoff with qualitative constraint is undecidable for POMDPs.

\item \emph{(Almost-sure winning with quantitative constraint).}
In contrast to our decidability result under finite-memory strategies 
for qualitative constraint, we show that the almost-sure winning problem for 
limit-average payoff with quantitative constraint is undecidable 
even for finite-memory strategies for POMDPs.

\end{enumerate}
In summary we establish the precise decidability frontier for POMDPs with 
limit-average payoff under probabilistic semantics (see Table~\ref{tab:almost_str}). 
For practical purposes, the most prominent question is the problem of
finite-memory strategies, and for finite-memory strategies we establish 
decidability with optimal EXPTIME-complete complexity for the important 
special case of qualitative constraint under almost-sure semantics.

\smallskip\noindent{\em Technical contributions.}
The key technical contribution for the decidability result is as follows.
Since belief-based strategies are not sufficient, standard subset 
construction techniques do not work.
For an arbitrary finite-memory strategy we construct a \emph{collapsed}
strategy that collapses memory states based on a graph construction given the 
strategy. 
The collapsed strategy at a collapsed memory state plays uniformly 
over actions that were played at all the corresponding memory states 
of the original strategy.
The main challenge is to show that the exponential size collapsed strategy, 
even though has less memory elements, does not destroy the structure 
of the recurrent classes of the original strategy.
For the computational complexity result, we show how to construct an
exponential size special class of POMDPs (which we call belief-observation
POMDPs where the belief is always the current observation) and
present polynomial time algorithms for the solution of the special 
belief-observation POMDPs of our construction.
For undecidability of almost-sure winning for qualitative constraint
under infinite-memory strategies we present a reduction from the 
the value~1 problem for PFA;
and for undecidability of almost-sure winning for quantitative 
constraint under finite-memory strategies we present a reduction
from the strict emptiness problem for PFA under probabilistic semantics.

\begin{table}[h]
\centering
\resizebox{12cm}{!}{
\begin{tabular}{|c|c|c|c|}
\toprule
& \multicolumn{2}{c|}{Almost-sure semantics} & Probabilistic semantics\\
\cmidrule{2-4}
& Fin. mem. & Inf. mem. & Fin./Inf. mem.\\
\midrule
PFA & PSPACE-c & PSPACE-c & Undec.\\
POMDP Qual. Constr.  & \textbf{EXPTIME-c} & \textbf{Undec.} & Undec.\\
POMDP Quan. Constr. & \textbf{Undec.} & \textbf{Undec.} & Undec.\\
\bottomrule

\end{tabular}
}
\caption{Computational complexity. New results are in bold fonts}
\label{tab:almost_str}
\end{table}

\newcommand{\C}{\mathcal{C}}

\section{Definitions}
In this section we present the definitions of POMDPs, 
strategies, objectives, and other basic definitions required 
throughout this work.
We follow the standard definitions of MDPs and POMDPs~\cite{Puterman,LittmanThesis}.

\smallskip\noindent{\bf Notations.}
Given a finite set $X$, we denote by $\powset(X)$ the set of subsets of $X$,
i.e., $\powset(X)$ is the power set of $X$.
A probability distribution $f$ on $X$ is a function $f:X \to [0,1]$ such 
that $\sum_{x\in X} f(x)=1$, and we denote by  $\distr(X)$ the set of 
all probability distributions on $X$. For $f \in \distr(X)$ we denote by $\supp(f)=\set{x\in X \mid f(x)>0}$
the support of $f$.

\begin{definition}[POMDP]
A \emph{Partially Observable Markov Decision Process (POMDP)} is a 
tuple $\game=(S,\act,\trans,\obs,\obsmap,s_0)$ where:
\begin{itemize}
 \item $S$ is a finite set of states;
 \item $\act$ is a finite alphabet of \emph{actions};
 \item $\trans:S\times\act \rightarrow \distr(S)$ is a 
 \emph{probabilistic transition function} that given a state $s$ and an
 action $a \in \act$ gives the probability distribution over the successor 
 states, i.e., $\trans(s,a)(s')$ denotes the transition probability from state
 $s$ to state $s'$ given action $a$; 
 \item $\obs$ is a finite set of \emph{observations}; 
 \item $\obsmap:S\rightarrow \obs$ is an \emph{observation function} that 
  maps every state to an observation; and 
 \item $s_0$ is the initial state. 
\end{itemize}
\end{definition}
\noindent
Given $s,s'\in S$ and $a\in\act$, we also write $\trans(s'|s,a)$ for 
$\trans(s,a)(s')$.
A state $s$ is \emph{absorbing} if for all actions $a$ we have 
$\trans(s,a)(s)=1$ (i.e., $s$ is never left from $s$).
For an observation $o$, 
we denote by $\obsmap^{-1}(o)=\set{s \in S \mid \obsmap(s)=o}$
the set of states with observation $o$.
For a set $U \subseteq S$ of states and $O \subseteq \obs$ of observations 
we denote 
$\obsmap(U)=\set{o \in \obs \mid \exists s \in U. \ \obsmap(s)=o}$ 
and $\obsmap^{-1}(O)= \bigcup_{o \in O} \obsmap^{-1}(o)$.

\begin{remark}
For technical convenience we assume that there is a unique initial 
state $s_0$ and we will also assume that the initial state has a unique 
observation, i.e., $|\obsmap^{-1}(\obsmap(s_0))|=1$.
In general there is an initial distribution $\alpha$ over initial states
that all have the same observation, i.e., $\supp(\alpha) \subseteq 
\obsmap^{-1}(o)$, for some $o \in \obs$.
However, this can be modeled easily by adding a new initial state 
$s_{\mathit{new}}$ with a unique observation such that in the first step gives 
the desired initial probability distribution $\alpha$, 
i.e., $\trans(s_{\mathit{new}},a)=\alpha$ for all actions $a \in \act$.
Hence for simplicity we assume there is a unique initial state $s_0$ with a
unique observation.
\end{remark}

\smallskip\noindent{\bf Plays, cones, and belief-updates.}
A \emph{play} (or a path) in a POMDP is an 
infinite sequence $(s_0,a_0,s_1,a_1,s_2,a_2,\ldots)$ of states and 
actions such that for all $i \geq 0$ we have $\trans(s_i,a_i)(s_{i+1})>0$.  
We write $\Omega$ for the set of all plays.
For a finite prefix $w \in (S\cdot A)^* \cdot S$ of a play, we denote by $\Cone(w)$ the 
set of plays with $w$ as the prefix (i.e., the cone or cylinder of the prefix $w$), 
and denote by $\last(w)$ the last state of $w$.
For a finite prefix $w=(s_0,a_0,s_1,a_1,\ldots,s_n)$ 
we denote by 
$\obsmap(w)=(\obsmap(s_0),a_0,\obsmap(s_1),a_1,\ldots,\obsmap(s_n))$ 
the observation and action sequence associated with $w$.
For a finite sequence $\rho=(o_0,a_0,o_1,a_1,\ldots,o_n)$ of observations and actions, the \emph{belief} $\belief(\rho)$ 
after the prefix $\rho$ is the set of states in which a finite prefix 
of a play can be after the sequence $\rho$ of observations and actions, 
i.e., $\belief(\rho)=\set{s_n=\last(w) \mid w=(s_0,a_0,s_1,a_1,\ldots,s_n), w \mbox{
is a prefix of a play, and for all } 0\leq i \leq n. \; \obsmap(s_i)=o_i}$. 
The belief-updates associated with finite-prefixes are as follows:
for prefixes $w$ and $w'=w \cdot a \cdot s$ the belief update is 
defined inductively as 
$\belief(\obsmap(w')) = 
\left(\bigcup_{s_1 \in \belief(\obsmap(w))} \supp(\trans(s_1,a)) \right)
\cap \obsmap^{-1}(\obsmap(s))$, 
i.e., the set
$\left(\bigcup_{s_1 \in \belief(\obsmap(w))} \supp(\trans(s_1,a)) \right)$
denotes the possible successors given the belief $\belief(\obsmap(w))$ and 
action $a$, and then the intersection with the set of states with the 
current observation $\gamma(s)$ gives the new belief set.

\smallskip\noindent{\bf Strategies.}
A \emph{strategy (or a policy)} is a recipe to extend prefixes of plays and 
is a function $\sigma: (S\cdot A)^* \cdot S \to \distr(A)$ that given a finite 
history (i.e., a finite prefix of a play) selects a probability distribution 
over the actions.
Since we consider POMDPs, strategies are \emph{observation-based}, i.e., 
for all histories $w=(s_0,a_0,s_1,a_1,\ldots,a_{n-1},s_n)$ and 
$w'=(s_0',a_0,s_1',a_1,\ldots,a_{n-1},s_n')$ such that for all 
$0\leq i \leq n$ we have $\obsmap(s_i)=\obsmap(s_i')$ (i.e., $\obsmap(w) = \obsmap(w')$), we must have 
$\sigma(w)=\sigma(w')$.
In other words, if the observation sequence is the same, then the strategy 
cannot distinguish between the prefixes and must play the same. 
We now present an equivalent definition of observation-based strategies  
such that the memory of the strategy is explicitly specified, and 
will be required to present finite-memory strategies.

\begin{definition}[Strategies with memory and finite-memory strategies]
A \emph{strategy} with memory is a tuple $\sigma=(\sigma_u,\sigma_n,M,m_0)$ 
where:
\begin{itemize}
 \item \emph{(Memory set).} $M$ is a denumerable set (finite or infinite) of memory elements (or memory states).
 \item \emph{(Action selection function).} The function $\sigma_n:M\rightarrow \distr(\act)$ is the 
	\emph{action selection function} that given the current memory 
	state gives the probability distribution over actions.
 \item \emph{(Memory update function).} The function $\sigma_u:M\times\obs\times\act\rightarrow \distr(M)$ 
 is the \emph{memory update function} that given the current memory state, 
 the current observation and action, updates the memory state probabilistically.
 \item \emph{(Initial memory).} The memory state $m_0\in M$ is the initial memory state.
\end{itemize}
A strategy is a \emph{finite-memory} strategy if the set $M$ of memory elements is finite.
A strategy is \emph{pure (or deterministic)} if the memory update function 
and the action selection function are deterministic, i.e., 
$\sigma_u: M \times \obs \times \act \to M$ and $\sigma_n: M \to \act$.
A strategy is \emph{memoryless (or stationary)} if it is independent of the 
history but depends only on the current observation, and can be represented
as a function $\sigma: \obs \to \distr(\act)$.  
\end{definition}

\smallskip\noindent{\bf Probability measure.}
Given a strategy $\sigma$, the unique probability measure obtained given $\sigma$ is denoted as $\mathbb{P}^{\sigma}(\cdot)$.
We first define the measure $\mu^\sigma(\cdot)$ on cones.
For $w=s_0$ we have $\mu^\sigma(\Cone(w))=1$, and 
for $w=s$ where $s\neq s_0$ we have  $\mu^\sigma(\Cone(w))=0$; and 
for $w' = w \cdot a\cdot s$ 
we have 
$\mu^\sigma(\Cone(w'))= \mu^\sigma(\Cone(w)) \cdot \sigma(w)(a) \cdot \trans(\last(w),a)(s)$. 
By Carathe\'odary's extension theorem, the function $\mu^\sigma(\cdot)$
can be uniquely extended to a probability measure $\prb^{\sigma}(\cdot)$
over Borel sets of infinite plays~\cite{Billingsley}.

\smallskip\noindent{\bf Objectives.}
An \emph{objective} in a POMDP $G$ is a measurable set $\varphi \subseteq \Omega$ 
of plays.
We first define \emph{limit-average payoff} (aka mean-payoff) function.
Given a POMDP we consider a reward function $\reward : \states \times \act \rightarrow [0,1]$
that maps every state action pair to a real-valued reward in the interval $[0,1]$. 
The $\limavg$ payoff function maps every play to a real-valued reward 
that is the long-run average of the rewards of the play.
Formally, given a play $\rho=(s_0,a_0,s_1,a_1,s_2,a_2,\ldots)$ we have
$$\limavg(\reward,\rho) = \liminf_{n \rightarrow \infty} \frac{1}{n} \cdot 
\sum_{i=0}^{n} \reward(s_i,a_i).$$
When the reward function $\reward$ is clear from the context, we drop it for 
simplicity.
For a reward function $\reward$, we consider two types of limit-average payoff
constraints.

\begin{enumerate}

\item \emph{Qualitative constraint.} 
The \emph{qualitative constraint} limit-average objective $\limavgone$ defines 
the set of paths such that the limit-average payoff is~1; i.e., 
$\limavgone= \set{\rho \mid \limavg(\rho)=1}$.

\item \emph{Quantitative constraints.}
Given a threshold $\lambda_1 \in (0,1)$, the \emph{quantitative constraint} 
limit-average objective $\limavg_{>\lambda_1}$ defines the set of paths such 
that the limit-average payoff is strictly greater than $\lambda_1$; i.e., 
$\limavg_{>\lambda_1}= \set{\rho \mid \limavg(\rho) > \lambda_1}$.

\end{enumerate}

\smallskip\noindent{\bf Probabilistic and almost-sure winning.}
Given a POMDP, an objective $\varphi$, and a class $\C$ of strategies, we say that:

\begin{itemize}
\item a strategy $\sigma \in \C$ is \emph{almost-sure winning} 
if $\prb^{\sigma}(\varphi) = 1$;



\item a strategy $\sigma \in \C$ is \emph{probabilistic winning}, for a 
threshold $\lambda_2 \in (0,1)$, if $\prb^{\sigma}(\varphi) \geq \lambda_2$.

\end{itemize}

\begin{theorem}[Results for PFA (probabilistic automata over finite words)~\cite{PazBook}]
The following assertions hold for the class $\C$ of all infinite-memory 
as well as finite-memory strategies:
(1)~the probabilistic winning problem is undecidable for PFA;
and
(2)~the almost-sure winning problem is PSPACE-complete for PFA.
\end{theorem}

Since PFA are a special case of POMDPs, the undecidability of the probabilistic 
winning problem for PFA implies the undecidability of the probabilistic 
winning problem for POMDPs with both qualitative and quantitative constraint 
limit-average objectives.
The almost-sure winning problem is PSPACE-complete for PFAs, and 
we study the complexity of the almost-sure winning problem for 
POMDPs with both qualitative and quantitative constraint limit-average objectives, 
under infinite-memory and finite-memory strategies.

\smallskip\noindent{\bf Basic properties of Markov Chains.}
Since our proofs will use results related to Markov chains, 
we start with some basic definitions and properties related to Markov chains.

\smallskip\noindent{\em Markov chains and recurrent classes.}
A Markov chain $\ov{\game}=(\ov{S},\ov{\trans})$ consists of a finite set $\ov{S}$ of 
states and a probabilistic transition function $\ov{\trans}:\ov{S} \rightarrow \distr(\ov{S})$.
Given the Markov chain, we consider the directed graph $(\ov{S},\ov{E})$ where $\ov{E}=\set{(\ov{s},\ov{s}') 
\mid \trans(\ov{s}' \mid \ov{s}) >0}$.
A \emph{recurrent class} $\ov{C} \subseteq \ov{S}$ of the Markov chain is a bottom 
strongly connected component (scc) in the graph  $(\ov{S},\ov{E})$ 
(a bottom scc is an scc with no edges out of the scc).
We denote by $\Rec(\ov{\game})$ the set of recurrent classes of the Markov chain,
i.e., $\Rec(\ov{\game})=\set{\ov{C} \mid \ov{C} \text{ is a recurrent class}}$.
Given a state $\ov{s}$ and a set $\ov{U}$ of states, we say that $\ov{U}$ is
reachable from $\ov{s}$ if there is a path from $\ov{s}$ to some state in 
$\ov{U}$ in the graph $(\ov{S},\ov{E})$. 
Given a state $\ov{s}$ of the Markov chain we denote by $\Rec(\ov{\game})(\ov{s}) \subseteq \Rec(\ov{\game})$ 
the subset of the recurrent classes reachable from $\ov{s}$ in $\ov{G}$.
A state is \emph{recurrent} if it belongs to a recurrent class.
The following standard properties of reachability and the recurrent classes 
will be used in our proofs:

\begin{itemize}
\item \emph{Property~1.}
(a)~For a set $\ov{T} \subseteq \ov{S}$, if for all states $\ov{s} \in \ov{S}$ 
there is a path to $\ov{T}$ (i.e., for all states there is a positive
probability to reach $\ov{T}$), then from all states the set $\ov{T}$
is reached with probability~1.
(b)~For all states $\ov{s}$, if the Markov chain starts at $\ov{s}$, then 
the set $\ov{{\mathsf C}}=\bigcup_{\ov{C} \in \Rec(\ov{\game})(\ov{s})} \ov{C}$ 
is reached with probability~1, 
i.e., the set of recurrent classes is reached with probability~1.
\item \emph{Property~2.}
For a recurrent class $\ov{C}$, for all states $\ov{s} \in \ov{C}$,
if the Markov chain starts at $\ov{s}$, then for all states
$\ov{t}\in \ov{C}$ the state $\ov{t}$ is visited infinitely 
often with probability~1, and is visited with positive average frequency 
(i.e., positive limit-average frequency) with probability~1.
\end{itemize}
The following lemma is an easy consequence of the above properties.

\begin{lemma}
\label{lem:rec_class}
Let $\ov{\game} = (\ov{S},\ov{\trans})$ be a Markov chain with a reward 
function $\reward: \ov{S}\rightarrow [0,1]$, and $\ov{s} \in \ov{S}$ a state 
of the Markov chain. 
The state $\ov{s}$ is almost-sure winning for the objective $\limavgone$ iff 
for all recurrent classes $\ov{C} \in \Rec(\ov{\game})(\ov{s})$ and for all 
states $\ov{s}_1 \in \ov{C}$ we have $\reward(\ov{s}_1) = 1$.
\end{lemma}

\smallskip\noindent{\em Markov chains under finite memory strategies.}
We now define Markov chains obtained by fixing  a finite-memory strategy in a POMDP
$\game$. 
A finite-memory strategy $\straa=(\straa_u,\straa_n,M,m_0)$ induces a Markov chain 
$(S\times M,\trans_{\straa})$, denoted $\game \restr_\straa$, 
with the probabilistic transition function $\trans_\straa: S\times M \rightarrow \distr(S\times M)$: 
given $s,s'\in S$ and $m,m'\in M$, the transition $\trans_\straa\big((s',m')\ |\ (s,m)\big)$ is the probability to go from state 
$(s,m)$ to state $(s',m')$ in one step under the strategy $\straa$. 
The probability of transition can be decomposed as follows:
\begin{itemize}
 \item First an action $a\in\act$ is sampled according to the distribution $\straa_n(m)$; 
 \item then the next state $s'$ is sampled according to the distribution $\trans(s,a)$; and
 \item finally the new memory $m'$ is sampled according to the distribution $\straa_u(m,\obsmap(s'),a)$
 (i.e., the new memory is sampled according $\straa_u$ given the old memory, new observation and the action).
\end{itemize}
More formally, we have:
\[
\trans_\sigma\big((s',m')\ |\
(s,m)\big)=\sum_{a\in\act}\straa_n(m)(a)\cdot\trans(s,a)(s')\cdot\straa_u(m,\obsmap(s'),a)(m').
\]
Given $s\in S$ and $m\in M$, we write $(G \restr_\straa)_{(s,m)}$ for the finite state 
Markov chain induced on $S\times M$ by the transition function 
$\trans_{\straa}$, given the initial state is $(s,m)$.

\newcommand{\Win}{\mathsf{Win}}

\section{Finite-memory strategies with Qualitative Constraint}
In this section we will establish the following three results 
for finite-memory strategies:
(i)~we show that in POMDPs with $\limavgone$ objectives 
belief-based strategies are not sufficient for almost-sure 
winning;
(ii)~we establish an exponential upper bound on the memory 
required by an almost-sure winning strategy for $\limavgone$ 
objectives;
and (iii)~we show that the decision problem is EXPTIME-complete.

\subsection{Belief is not sufficient}
We now show with an example that there exist POMDPs with $\limavgone$ 
objectives, where finite-memory randomized almost-sure winning strategies 
exist, but there exists no belief-based randomized almost-sure winning 
strategy (a belief-based strategy only uses memory that relies on the 
subset construction where the subset denotes the possible current states 
called belief).
In our example we will present the counter-example even for POMDPs with 
restricted reward function $\reward$ assigning only Boolean rewards $0$ and $1$ 
to the states (the reward does not depend on the action played but only on the 
states of the POMDP).

	\begin{figure}
		\begin{center}
		\resizebox{\linewidth}{!}{
		\begin{tikzpicture}[>=latex]
		\node[State,text=] (x) {$X$};
		\node[State,right of=x,xshift=50] (x') {$X'$};
		\node[State,below of=x,yshift=-30] (y) {$Y$}; 
		\node[State,below of=x',yshift=-30] (y') {$Y'$}; 
		\node[State,below of=y,yshift=-30] (z) {$Z$}; 
		\node[State,below of=y',yshift=-30] (z') {$Z'$}; 
		\draw[->]
		(x) edge[bend left] node[above] {a} (x')
		(x) edge[bend right] node[left] {b} (y)
		(x') edge[bend left] node[above] {b} (x)
		(x') edge[bend left] node[right] {a} (y')
		(z) edge[bend left] node[left] {a} (y)
		(z) edge[bend right] node[above] {b} (z')
		(z') edge[bend right] node[above] {a} (z)
		(z') edge[bend right] node[right] {b} (y')
		(y) edge[bend right,out=270] node[right] {a,b} (x)
		(y') edge[bend left,out=90] node[left] {a,b} (x')
		(y) edge[bend left,out=90] node[right] {a,b} (z)
		(y') edge[bend right,out=270] node[left] {a,b} (z')
		;
		\node[right of=y,xshift=-5] {$\frac{1}{2}$};
		\draw ($(y) +(18pt, 5pt)$) arc (20:-20:0.5cm);
		\draw ($(y') +(-18pt, -5pt)$) arc (-160:-200:0.5cm);
		\node[left of=y',xshift=5] {$\frac{1}{2}$};
		\node[fit=(x) (x') (y) (y') (z) (z'),inner sep=0.8cm] (mdp) {};
		\node[above of=mdp,yshift=65] {POMDP $G$};
		\node[State,right of=x',xshift=4em] (x1) {$X$};
		\node[State,right of=x1,xshift=50] (x1') {$X'$};
		\node[State,below of=x1,yshift=-30] (y1) {$Y$}; 
		\node[State,below of=x1',yshift=-30] (y1') {$Y'$}; 
		\node[State,below of=y1,yshift=-30] (z1) {$Z$}; 
		\node[State,below of=y1',yshift=-30] (z1') {$Z'$}; 
		\draw[->]
		(x1) edge[bend left] node[above] {} (x1')
		(x1') edge[bend left] node[right] {} (y1')
		(z1) edge[bend left] node[left] {} (y1)
		(z1') edge[bend right] node[above] {} (z1)
		(y1) edge[bend right,out=270] node[right] {} (x1)
		(y1') edge[bend left,out=90] node[left] {} (x1')
		(y1) edge[bend left,out=90] node[right] {} (z1)
		(y1') edge[bend right,out=270] node[left] {} (z1')
		;
		
		\draw ($(y1) +(18pt, 5pt)$) arc (20:-20:0.5cm);
		\draw ($(y1') +(-18pt, -5pt)$) arc (-160:-200:0.5cm);
		\node[right of=y1,xshift=-5] {$\frac{1}{2}$};
		\node[left of=y1',xshift=5] {$\frac{1}{2}$};
		\node[rectangle,fit= (x1) (x1') (y1) (y1') (z1) (z1'),inner sep=0.5cm,draw=black, dashed] (rec1) {};
		\node[above of=rec1,yshift=65] {MC $G\restr_{\sigma_1}$};
		\node[below of=rec1,yshift=-62,xshift=7] {\emph{Rec}: $\{X,X',Y,Y',Z,Z'\}$};

		\node[State,right of=x1',xshift=4em] (x2) {$X$};
		\node[State,right of=x2,xshift=50] (x2') {$X'$};
		\node[State,below of=x2,yshift=-30] (y2) {$Y$}; 
		\node[State,below of=x2',yshift=-30] (y2') {$Y'$}; 
		\node[State,below of=y2,yshift=-30] (z2) {$Z$}; 
		\node[State,below of=y2',yshift=-30] (z2') {$Z'$}; 
		\draw[->]
		(x2) edge[bend right] node[left] {} (y2)
		(x2') edge[bend left] node[above] {} (x2)
		(z2) edge[bend right] node[above] {} (z2')
		(z2') edge[bend right] node[right] {} (y2')
		(y2) edge[bend right,out=270] node[right] {} (x2)
		(y2') edge[bend left,out=90] node[left] {} (x2')
		(y2) edge[bend left,out=90] node[right] {} (z2)
		(y2') edge[bend right,out=270] node[left] {} (z2')
		;
		\draw ($(y2) +(18pt, 5pt)$) arc (20:-20:0.5cm);
		\draw ($(y2') +(-18pt, -5pt)$) arc (-160:-200:0.5cm);
		\node[right of=y2,xshift=-5] {$\frac{1}{2}$};
		\node[left of=y2',xshift=5] {$\frac{1}{2}$};
		\node[rectangle,fit= (x2) (x2') (y2) (y2') (z2) (z2') ,inner sep=0.5cm,draw=black, dashed] (rec2) {};
		\node[below of=rec2,yshift=-62,xshift=7] {\emph{Rec}: $\{X,X',Y,Y',Z,Z'\}$};
		\node[above of=rec2,yshift=65] {MC $G\restr_{\sigma_2}$};

		\end{tikzpicture}
		}
		\end{center}
		\caption{Belief is not sufficient}
		\label{fig:simple_example}
		\end{figure}
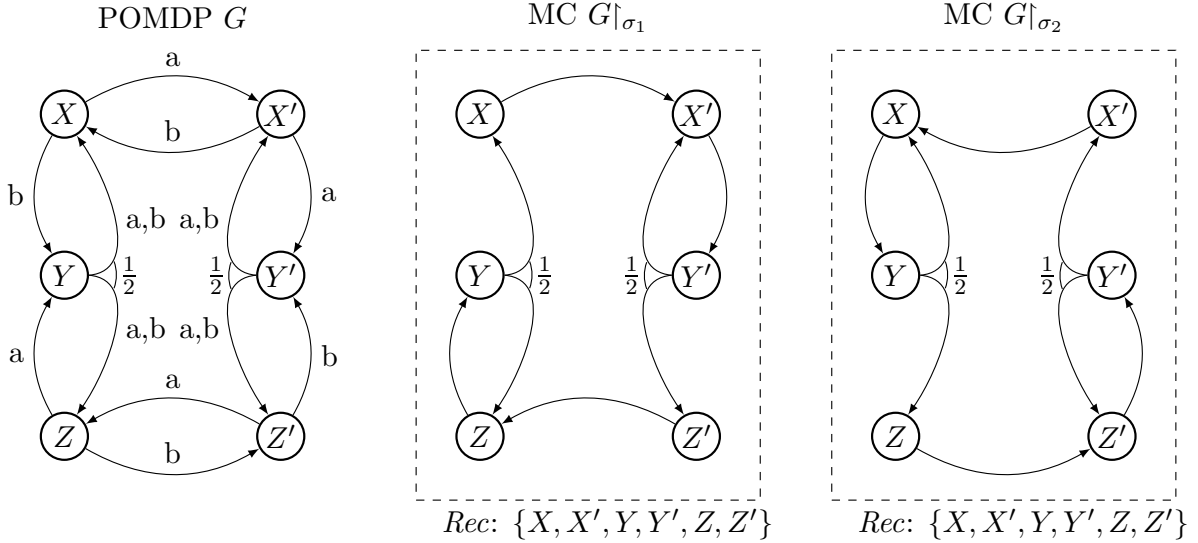

	\begin{example}\label{ex:belief-suff}
	We consider a POMDP with state space $\set{s_0,X,X',Y,Y',Z,Z'}$ and
	action set $\set{a,b}$, and let $U= \set{X,X',Y,Y',Z,Z'}$.
	From the initial state $s_0$ all the other states are reached with 
	uniform probability in one-step, i.e., for all $s' \in U=\set{X,X',Y,Y',Z,Z'}$ 
	we have $\trans(s_0,a)(s')=\trans(s_0,b)(s')=\frac{1}{6}$.
	The transitions from the other states are as follows 
	(shown in Figure~\ref{fig:simple_example}): 
	(i)~$\trans(X,a)(X')=1$ and $\trans(X,b)(Y)=1$;
	(ii)~$\trans(X',a)(Y')=1$ and $\trans(X',b)(X)=1$;
	(iii)~$\trans(Z,a)(Y)=1$ and $\trans(Z,b)(Z')=1$;
	(iv)~$\trans(Z',a)(Z)=1$ and $\trans(Z',b)(Y')=1$;
	(v)~$\trans(Y,a)(X)=\trans(Y,b)(X)=\trans(Y,a)(Z)=\trans(Y,b)(Z)=\frac{1}{2}$;
	and
	(vi)~$\trans(Y',a)(X')=\trans(Y',b)(X')=\trans(Y',a)(Z')=\trans(Y',b)(Z')=\frac{1}{2}$.
	All states in $U$ have the same observation.
	The reward function $\reward$ assigns the reward $1$ to states $X,X',Z,Z'$ and reward $0$ to states $Y$ and $Y'$.

	The belief initially after one-step is the set $U=\set{X,X',Y,Y',Z,Z'}$ since
	from $s_0$ all of them are reached with positive probability.
	The belief is always the set $U$ since every state has an input edge for 
	every action, i.e., if the current belief is $U$ (i.e., the set of states 
	that the POMDP is currently in with positive probability is $U$), then
	irrespective of whether $a$ or $b$ is chosen all states of $U$ are reached
	with positive probability and hence the belief set is again $U$.
	There are three belief-based strategies: (i)~$\sigma_1$ that plays always $a$;
	(ii)~$\sigma_2$ that plays always $b$; or (iii)~$\sigma_3$ that plays 
	both $a$ and $b$ with positive probability.
	The Markov chains $G\restr_{\sigma_1}$ (resp. $\sigma_2$ and $\sigma_3$) 
	are obtained by retaining the edges labeled by action $a$ (resp. action $b$, 
	and both actions $a$ and $b$).
	For all the three strategies, the Markov chains obtained contain the whole 
	set $U$ as the reachable recurrent class. It follows that in all the Markov 
	chains there exists a reachable recurrent class containing a state with reward 
	$0$, and by Lemma~\ref{lem:rec_class} none of the belief-based strategies 
	$\sigma_1, \sigma_2$ or $\sigma_3$ are almost-sure winning for the 
	$\limavgone$ objective.

	The Markov chains $G \restr_{\sigma_1}$ and $G \restr_{\sigma_2}$ are also 
	shown in Figure~\ref{fig:simple_example}, and the graph of 
	$G \restr_{\sigma_3}$ is the same as the POMDP $G$ (with edge labels removed).
	The strategy $\sigma_4$ that plays action $a$ and $b$ alternately gives 
	rise to the Markov chain $G \restr_{\sigma_4}$ (shown in 
	Figure~\ref{fig:simple_example1}) where the recurrent classes do not intersect
	with $Y$ or $Y'$, and is a finite-memory almost-sure winning strategy
	for the $\limavgone$ objective.
	\qed
	\end{example}

	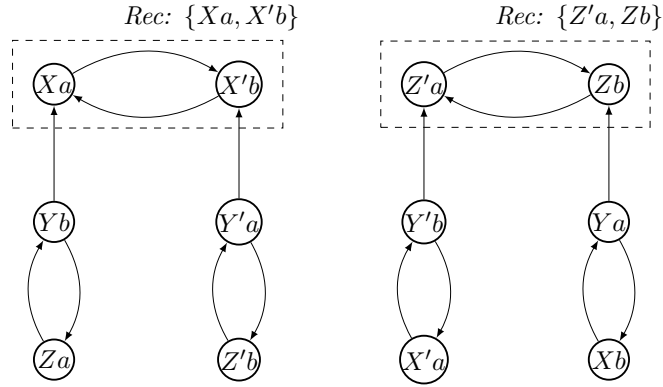
\begin{figure}[h]
	\centering
	\resizebox{9cm}{!}{
	\begin{tikzpicture}[>=latex]
		\node[State,text=] (xa) {$Xa$};
		\node[State,right of=xa,xshift=50] (x'b) {$X'b$};
		\node[State,below of=xa,yshift=-30] (yb) {$Yb$};
		\node[State,below of=x'b,yshift=-30] (y'a) {$Y'a$};
		\node[State,below of=yb,yshift=-30] (za) {$Za$};
		\node[State,below of=y'a,yshift=-30] (z'b) {$Z'b$};
		\draw[->]
		(xa) edge[bend left] node [above] {} (x'b)
		(x'b) edge[bend left] node [above] {} (xa)
		(yb) edge[] node [above] {} (xa)
		(y'a) edge[] node [above] {} (x'b)
		(za) edge[bend left] node [above] {} (yb)
		(yb) edge[bend left] node [above] {} (za)
		(y'a) edge[bend left] node [above] {} (z'b)
		(z'b) edge[bend left] node [above] {} (y'a)
		;
		\node[rectangle,fit= (xa) (x'b) ,inner sep=0.3cm,draw=black, dashed] (rec1) {};
		\node[above of=rec1,xshift=27] {\emph{Rec:} $\{Xa,X'b\}$};
		
		\node[State,right of=x'b, xshift=50] (z'a) {$Z'a$};
		\node[State,right of=z'a,xshift=50] (zb) {$Zb$};
		\node[State,below of=z'a,yshift=-30] (y'b) {$Y'b$};
		\node[State,below of=zb,yshift=-30] (ya) {$Ya$};
		\node[State,below of=y'b,yshift=-30] (x'a) {$X'a$};
		\node[State,below of=ya,yshift=-30] (xb) {$Xb$};
		\draw[->]
		(z'a) edge[bend left] node [above] {} (zb)
		(zb) edge[bend left] node [above] {} (z'a)
		(y'b) edge[] node [above] {} (z'a)
		(ya) edge[] node [above] {} (zb)
		(y'b) edge[bend left] node [above] {} (x'a)
		(x'a) edge[bend left] node [above] {} (y'b)
		(ya) edge[bend left] node [above] {} (xb)
		(xb) edge[bend left] node [above] {} (ya)
		;
		\node[rectangle,fit= (z'a) (zb) ,inner sep=0.3cm,draw=black, dashed] (rec2) {};
		\node[above of=rec2,xshift=27] {\emph{Rec:} $\{Z'a,Zb\}$};

	\end{tikzpicture}
	}
	\caption{The Markov chain $G \restr_{\sigma_4}$.}
	\label{fig:simple_example1}
	\end{figure}

	In Example~\ref{ex:belief-suff} the $\limavgone$ objective is not observation-based, i.e., 
	there are states within the same observation with different rewards.
	In the following example we modify Example~\ref{ex:belief-suff} to show that 
	randomized belief-based strategies are not sufficient even if we consider observation-based 
	$\limavgone$ objectives, i.e, all states within the same observation are required to have the same reward.

	\begin{example}\label{ex:complex}
	We consider the POMDP shown in Figure~\ref{fig:complex_example}:
	the transition edges in the set $U=\set{X,X',Y,Y',Z,Z'}$ are exactly
	the same as in Figure~\ref{fig:simple_example}, and the transition 
	probabilities are always uniform over the support set. 
	We add a new state $B$ and from the state $Y$ and $Y'$ add positive 
	transition probabilities (probability~$\frac{1}{3}$) to the state
	$B$ for both actions $a$ and $b$.
	Recall that $Y$ and $Y'$ were having reward $0$ in Example~\ref{ex:belief-suff}.
	From state $B$ all states in $U$ are reached with positive probability for both actions
	$a$ and $b$.
	All states in $U$ have the same observation (denoted as $o_U$), 
	and the state $B$ has a new and different observation (denoted as $o_B$). 
	All the states in observation $o_U$ are assigned reward $1$, and the state $B$ in observation $o_B$ is assigned reward $0$.
	Note that the objective is an observation-based $\limavgone$ objective.
	Since we retain all edges as in Figure~\ref{fig:simple_example} and
	from $B$ all states in $U$ are reached with positive probability
	in one step, whenever the current observation is $o_U$, then the belief
	is the set $U$.
	As in Example~\ref{ex:belief-suff} there are three belief-based
	strategies ($\straa_1,\straa_2$ and $\straa_3$) in belief $U$, 
	and the Markov chains under $\straa_1$ and $\straa_2$ are shown
	in Figure~\ref{fig:complex_example}, and the Markov chain under
	$\straa_3$ has the same edges as the original POMDP.
	For all the belief-based strategies the recurrent class
	contains the state $B$, i.e., a state with reward $0$, and by Lemma~\ref{lem:rec_class} are not almost-sure winning strategies for the $\limavgone$ 
	objective. The strategy $\straa_4$ that alternates actions $a$ and $b$ 
	is a finite-memory almost-sure winning strategy  for the $\limavgone$ 
	objective and the Markov chain obtained given $\straa_4$ is shown in 
	Figure~\ref{fig:complex_example1}.
	\end{example}

\begin{figure}
\begin{center}
\resizebox{\linewidth}{!}{
\begin{tikzpicture}[>=latex]
	\node[State,text=] (x) {$X$};
	\node[State,right of=x,xshift=100] (x') {$X'$};
	\node[State,below of=x,yshift=-30] (y) {$Y$}; 
	\node[State,below of=x',yshift=-30] (y') {$Y'$}; 
	\node[State,below of=y,yshift=-30] (z) {$Z$}; 
	\node[State,below of=y',yshift=-30] (z') {$Z'$}; 
	\node[RState,right of=y,xshift=35] (b) {$B$};
	\draw[->]
	(x) edge[bend left=25] node[above] {a} (x')
	(x) edge[bend right] node[left] {b} (y)
	(x') edge[bend left=25] node[above] {b} (x)
	(x') edge[bend left] node[right] {a} (y')
	(z) edge[bend left] node[left] {a} (y)
	(z) edge[bend right=25] node[above] {b} (z')
	(z') edge[bend right=25] node[above] {a} (z)
	(z') edge[bend right] node[right] {b} (y')
	(y) edge[bend right,out=270]   (x)
	(y) edge[] node[above] {a,b} (b)
	(y') edge[] node[above] {a,b} (b)
	(y') edge[bend left,out=90] (x')
	(y) edge[bend left,out=90]  (z)
	(y') edge[bend right,out=270] (z')
	(b) edge[red, very thick, bend left=50] (y')
	(b) edge[red, very thick, bend right=50] (y)
	(b) edge[red, very thick, bend right] (x)
	(b) edge[red, very thick, bend left] (x')
	(b) edge[red, very thick, bend right=15] (z)
	(b) edge[red, very thick, bend left=15] (z')
	;
	\node[right of=y,xshift=-3,yshift=-9] {$\frac{1}{3}$};
	\draw ($(y) +(18pt, 5pt)$) arc (20:-20:0.5cm);
	\draw ($(y') +(-18pt, -5pt)$) arc (-160:-200:0.5cm);
	\node[left of=y',xshift=3,yshift=-9] {$\frac{1}{3}$};
	\node[fit=(x) (x') (y) (y') (z) (z'),inner sep=0.8cm] (mdp) {};
	\node[above of=mdp,yshift=65] {POMDP $G$};

	\node[State,right of =x',xshift=4em] (x1) {$X$};
	\node[State,right of=x1,xshift=100] (x1') {$X'$};
	\node[State,below of=x1,yshift=-30] (y1) {$Y$}; 
	\node[State,below of=x1',yshift=-30] (y1') {$Y'$}; 
	\node[State,below of=y1,yshift=-30] (z1) {$Z$}; 
	\node[State,below of=y1',yshift=-30] (z1') {$Z'$}; 
	\node[RState,right of=y1,xshift=35] (b1) {$B$};
	\draw[->]
	(x1) edge[bend left=25] node[above] {} (x1')
	(x1') edge[bend left] node[right] {} (y1')
	(z1) edge[bend left] node[left] {} (y1)
	(z1') edge[bend right=25] node[above] {} (z1)
	(y1) edge[bend right,out=270]   (x1)
	(y1') edge[bend left,out=90] (x1')
	(y1) edge[bend left,out=90]  (z1)
	(y1') edge[bend right,out=270] (z1')
	(b1) edge[red, very thick, bend left=50] (y1')
	(b1) edge[red, very thick, bend right=50] (y1)
	(b1) edge[red, very thick, bend right] (x1)
	(b1) edge[red, very thick, bend left] (x1')
	(b1) edge[red, very thick, bend right=15] (z1)
	(b1) edge[red, very thick, bend left=15] (z1')
	(y1) edge[] node[above] {} (b1)
	(y1') edge[] node[above] {} (b1)
	;
	\node[right of=y1,xshift=-3,yshift=-9] {$\frac{1}{3}$};
	\draw ($(y1) +(18pt, 5pt)$) arc (20:-20:0.5cm);
	\draw ($(y1') +(-18pt, -5pt)$) arc (-160:-200:0.5cm);
	\node[left of=y1',xshift=3,yshift=-9] {$\frac{1}{3}$};
	\node[rectangle,fit= (x1) (x1') (y1) (y1') (z1) (z1') ,inner sep=0.5cm,draw=black, dashed] (straa1) {};
	\node[below of=straa1,yshift=-62,xshift=7] {\emph{Rec}: $\{X,X',Y,Y',Z,Z',B\}$};
	\node[above of=straa1, yshift=65] {MC $G\restr_{\sigma_1}$};

	\node[State,right of =x1',xshift=4em] (x2) {$X$};
	\node[State,right of=x2,xshift=100] (x2') {$X'$};
	\node[State,below of=x2,yshift=-30] (y2) {$Y$}; 
	\node[State,below of=x2',yshift=-30] (y2') {$Y'$}; 
	\node[State,below of=y2,yshift=-30] (z2) {$Z$}; 
	\node[State,below of=y2',yshift=-30] (z2') {$Z'$}; 
	\node[RState,right of=y2,xshift=35] (b2) {$B$};
	\draw[->]
	(x2) edge[bend right] node[left] {} (y2)
	(x2') edge[bend left=25] node[above] {} (x2)
	(z2) edge[bend right=25] node[above] {} (z2')
	(z2') edge[bend right] node[right] {} (y2')
	(y2) edge[bend right,out=270]   (x2)
	(y2) edge[] node[above] {} (b2)
	(y2') edge[] node[above] {} (b2)
	(y2') edge[bend left,out=90] (x2')
	(y2) edge[bend left,out=90]  (z2)
	(y2') edge[bend right,out=270] (z2')
	(b2) edge[red, very thick, bend left=50] (y2')
	(b2) edge[red, very thick, bend right=50] (y2)
	(b2) edge[red, very thick, bend right] (x2)
	(b2) edge[red, very thick, bend left] (x2')
	(b2) edge[red, very thick, bend right=15] (z2)
	(b2) edge[red, very thick, bend left=15] (z2')

	;
	\node[right of=y2,xshift=-3,yshift=-9] {$\frac{1}{3}$};
	\draw ($(y2) +(18pt, 5pt)$) arc (20:-20:0.5cm);
	\draw ($(y2') +(-18pt, -5pt)$) arc (-160:-200:0.5cm);
	\node[left of=y2',xshift=3,yshift=-9] {$\frac{1}{3}$};
	\node[rectangle,fit= (x2) (x2') (y2) (y2') (z2) (z2') ,inner sep=0.5cm,draw=black, dashed] (straa2) {};
	\node[below of=straa2,yshift=-62,xshift=7] {\emph{Rec}: $\{X,X',Y,Y',Z,Z',B\}$};
	\node[above of=straa2, yshift=65] {MC $G\restr_{\sigma_2}$};
\end{tikzpicture}
}
\caption{Belief is not sufficient}
\label{fig:complex_example}
\end{center}
\end{figure}
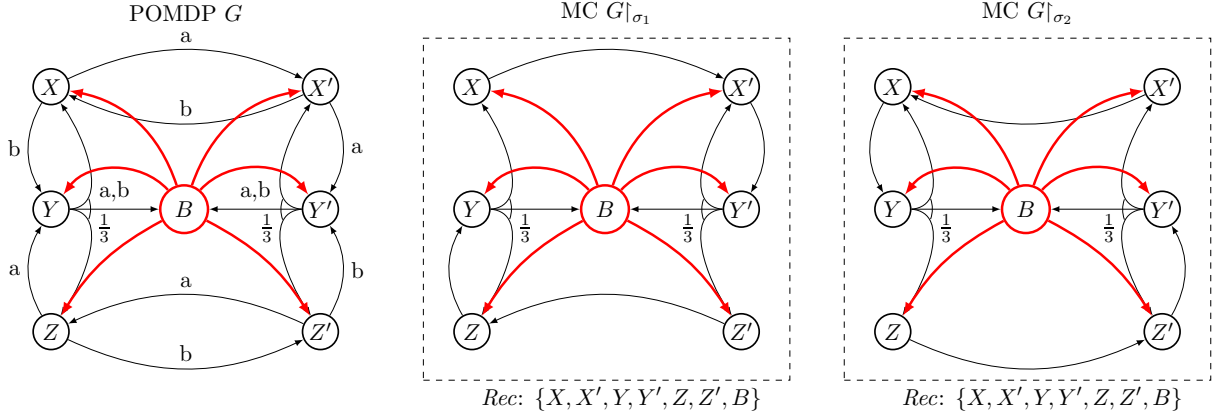

\begin{figure}[h]
\centering

\resizebox{13cm}{!}{
\begin{tikzpicture}[>=latex]

	\node[State,text=] (za) {$Za$};
	\node[State,right of=za,xshift=50] (yb) {$Yb$};
	\node[State,right of=yb,xshift=50] (xa) {$Xa$};
	\node[State,right of=xa,xshift=50] (x'b) {$X'b$};
	\node[State,right of=x'b,xshift=50] (y'a) {$Y'a$};
	\node[State,right of=y'a,xshift=50] (z'b) {$Z'b$};
	\draw[->]
	(xa) edge[bend left] node [above] {} (x'b)
	(x'b) edge[bend left] node [above] {} (xa)
	(yb) edge[] node [above] {} (xa)
	(y'a) edge[] node [above] {} (x'b)
	(za) edge[bend left] node [above] {} (yb)
	(yb) edge[bend left] node [above] {} (za)
	(y'a) edge[bend left] node [above] {} (z'b)
	(z'b) edge[bend left] node [above] {} (y'a)
	;
	\node[rectangle,fit= (xa) (x'b) ,inner sep=0.3cm,draw=black, dashed] (rec1) {};
	\node[above of=rec1,xshift=27] {\emph{Rec:} $\{Xa,X'b\}$};

	\node[State,below of=yb,yshift=-30] (ba) {$Ba$};
	\node[State,below of=y'a,yshift=-30] (bb) {$Bb$};
	\draw[->]
	(yb) edge[]  (ba)
	(y'a) edge[]  (bb)
	(ba) edge[bend left] (yb)
	(ba) edge[bend right=10] (x'b)
	(ba) edge[bend right=10] (z'b)
	(bb) edge[bend right] (y'a)
	(bb) edge[bend left=10] (xa)
	(bb) edge[bend left=10] (za)
	;
	\node[State,below of=za,yshift=-100] (x'a) {$X'a$};
	\node[State,right of=x'a,xshift=50] (y'b) {$Y'b$};
	\node[State,right of=y'b,xshift=50] (z'a) {$Z'a$};
	\node[State,right of=z'a,xshift=50] (zb) {$Zb$};
	\node[State,right of=zb,xshift=50] (ya) {$Ya$};
	\node[State,right of=ya,xshift=50] (xb) {$Xb$};
	\draw[->]
	(z'a) edge[bend left] node [above] {} (zb)
	(zb) edge[bend left] node [above] {} (z'a)
	(y'b) edge[] node [above] {} (z'a)
	(ya) edge[] node [above] {} (zb)
	(y'b) edge[bend left] node [above] {} (x'a)
	(x'a) edge[bend left] node [above] {} (y'b)
	(ya) edge[bend left] node [above] {} (xb)
	(xb) edge[bend left] node [above] {} (ya)
	;
	\node[rectangle,fit= (z'a) (zb) ,inner sep=0.3cm,draw=black, dashed] (rec2) {};
	\node[below of=rec2,xshift=27] {\emph{Rec:} $\{Z'a,Zb\}$};
	\draw[->]
	(y'b) edge[]  (ba)
	(ya) edge[]  (bb)
	(ba) edge[bend right] (y'b)
	(ba) edge[bend left=10] (xb)
	(ba) edge[bend left=10] (zb)
	(bb) edge[bend left] (ya)
	(bb) edge[bend right=10] (x'a)
	(bb) edge[bend right=10] (z'a)
	;
	\end{tikzpicture}
	}
	\caption{The Markov chain $G \restr_{\sigma_4}$.}
	\label{fig:complex_example1}
	\end{figure}
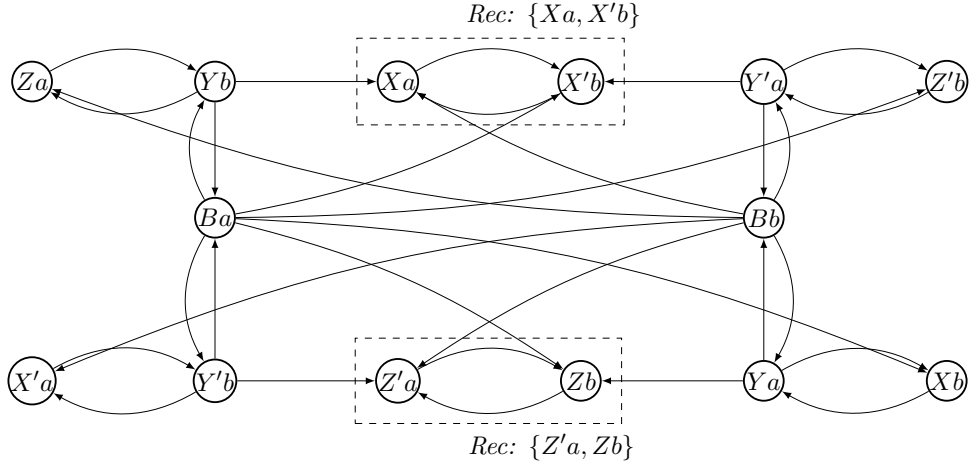

\subsection{Strategy complexity}
For the rest of the subsection we fix a finite-memory almost-sure winning strategy 
$\straa=(\straa_u,\straa_n,M,m_0)$ on the POMDP 
$\game=(\states,\act,\trans,\obs,\obsmap,s_0)$ with a reward function 
$\reward$ for the objective $\limavgone$.
Our goal is to construct an almost-sure winning strategy for the $\limavgone$ 
objective with memory size at most $\mem = 2^{3\cdot |S|} \cdot 2^{|\act|}$.
We start with a few definitions associated with strategy $\straa$.
For $m \in M$:

\begin{itemize}
\item  The function $\rec_\straa(m): \states \rightarrow \{0,1\}$ is 
such that $\rec_\straa(m)(s)$ is $1$ iff the state $(s,m)$ is recurrent in 
the Markov chain $\game \restr_\straa$ and $0$ otherwise. 
(The $\rec$ stands for recurrence function).

\item The function $\win_\straa(m): \states \rightarrow \{0,1\}$ is such 
that $\win_\straa(m)(s)$ is $1$ iff the state $(s,m)$ is almost-sure winning
for the $\limavgone$ objective in the Markov chain $\game \restr_\straa$ and 
$0$ otherwise.
(The $\win$ stands for almost-sure win function). 

\item We also consider $\actions_\straa(m) = \supp(\straa_n(m))$ that for
every memory element gives the support of the probability distribution over
actions played at $m$.
\end{itemize}

\begin{remark}\label{rem:win}
Let $(s',m')$ be a state reachable from $(s,m)$ in the Markov chain 
$\game\restr_\straa$. If the state $(s,m)$ is almost-sure winning
for the $\limavgone$ objective, then the state $(s',m')$ is also almost-sure 
winning  for the $\limavgone$ objective.
\end{remark}

\noindent{\bf Collapsed graph of $\straa$.}
Given the strategy $\straa$ we define the notion of a \emph{\projection graph} $\prgr(\straa) = (V,E)$, where the states of the graph are elements from the set $V = \{(Y,\win_\straa(m),\rec_\straa(m),\actions_\straa(m)) \mid Y \subseteq \states \text{ and } m \in M \}$ and the initial state is $(\{s_0\},\win_\straa(m_0), \rec_\straa(m_0), \actions_\straa(m_0))$. The edges in $E$ are labeled by actions in $\act$. There is an edge 
$$(Y,\win_\straa(m),\rec_\straa(m),\actions_\straa(m)) \stackrel{a}{\rightarrow} (Y',\win_\straa(m'),\rec_\straa(m'),\actions_\straa(m'))$$ 
in the \projection graph $\prgr(\straa)$ iff there exists an observation $o \in \obs$ such that 
\begin{enumerate}
\item the action $a \in \actions_\straa(m)$,
\item the set $Y'$ is non-empty and it is the belief update from $Y$, under action $a$ and the observation $o$, i.e., $Y' = \bigcup_{s\in Y}\supp(\trans(s,a)) \cap \obsmap^{-1}(o)$, and
\item $m' \in \supp (\straa_u(m,o,a))$.
\end{enumerate}
Note that the number of states in the graph is bounded by $|V| \leq \mem$.

In the following lemma we establish the connection of the functions 
$\rec_{\straa}(m)$ and $\win_{\straa}(m)$ with the edges of the 
\projection graph. 
Intuitively the lemma shows that when the function $\rec_{\straa}(m)$ 
(resp. $\win_{\straa}(m)$) is set to $1$ for a state $s$ of a vertex of the 
\projection graph, then for all successor vertices along the edges in the 
\projection graph, the function $\rec$ (resp. $\win$) is also set to $1$ for 
successors of state $s$.

\begin{lemma}
\label{lem:fix}
Let $(Y,W,R,A) \stackrel{a}{\rightarrow} (Y',W',R',A')$ be an edge in the 
\projection graph $\prgr(\straa) = (V,E)$. Then for all $s\in Y$ the 
following assertions hold:
\begin{enumerate}
 \item If $W(s) = 1$, then for all $s' \in \supp(\trans(s,a)) \cap Y'$ we have that $W'(s')=1$. 
 \item If $R(s) = 1$, then for all $s' \in \supp(\trans(s,a)) \cap Y'$ we have that $R'(s')=1$. 
\end{enumerate}
\end{lemma}
\begin{proof} 
We present proof of both the items below.
\begin{enumerate}
\item 
Let $(Y,W,R,A) \stackrel{a}{\rightarrow} (Y',W',R',A')$	be an edge 
in the \projection graph $\prgr(\straa) = (V,E)$ and a state $s \in Y$ 
such that $W(s) = 1$. 
It follows that there exist memories $m,m' \in M$ and an observation 
$o \in \obs$ such that (i)~$W = \win_\straa(m)$; (ii)~$W' = \win_\straa(m')$; 
(iii)~$a\in \supp(\straa_n(m))$; (iv)~$m' \in \supp(\straa_u(m,o,a))$; and finally (v) $\obs(s') = o$.
From all the points above it follows that there exists an edge $(s,m) \rightarrow (s',m')$ 
in the Markov chain $\game\restr_\straa$. 
As the state $(s,m)$ is almost-sure winning (since $W(s)=1$) 
it follows by Remark~\ref{rem:win} that the state $(s',m')$ must also be
almost-sure winning and therefore $W'(s') = 1$.

\item As in the proof of the first part we have that 
$(s',m')$ is reachable from $(s,m)$ in the Markov chain $\game\restr_\straa$. 
As every state reachable from a recurrent state in a Markov chain is also recurrent we have $R'(s')=1$.
\end{enumerate}
The desired result follows.
\end{proof}

We now define the \projected strategy for $\straa$. Intuitively we collapse 
memory elements of the original strategy $\straa$ whenever they agree on all 
the $\rec$, $\win$, and $\actions$ functions. 
The \projected strategy plays uniformly all the actions from the set given by
$\actions$ in the collapsed state.

\smallskip\noindent{\bf Collapsed strategy.}
We now construct the \emph{\projected strategy} 
$\straa'=(\straa'_u,\straa'_n,M',m'_0)$ of $\straa$ based on the 
\projection graph $\prgr(\straa)=(V,E)$. 
We will refer to this construction by $\straa'= \prst(\straa)$.
\begin{itemize}

\item The memory set $M'$ are the vertices of the \projection graph 
$\prgr(\straa) = (V,E)$, i.e.,  $M' = V = \{(Y,\win_\straa(m),\rec_\straa(m),\actions_\straa(m)) \mid Y \subseteq \states \text{ and } m \in M\}$.

\item The initial memory is $m'_0 = (\{s_0\}, \win_\straa(m_0), \rec_\straa(m_0), \actions_\straa(m_0))$.

\item The next action function given a memory $(Y,W,R,A) \in M'$ is the uniform distribution over the set of actions $\{a \mid \exists (Y',W',R',A') \in M' \text{ and } (Y,W,R,A) \stackrel{a}{\rightarrow} (Y',W',R',A') \in E\}$, where $E$ are the edges of the \projection graph.

\item The memory update function $\straa'_u((Y,W,R,A),o,a)$ given a memory element $(Y,W,R,A) \in M'$, $a \in \act$, and $o \in \obs$ is the uniform distribution over the set of states $\{(Y',W',R',A') \mid (Y,W,R,A) \stackrel{a}{\rightarrow} (Y', W',R',A') \in E \text{ and } Y' \subseteq \obsmap^{-1}(o)\}$.

\end{itemize}

The following lemma intuitively shows that the \projected strategy can reach all the states the original strategy could reach.

\begin{lemma}
\label{lem:reachability}
Let $\straa' = \prst(\straa)$ be the \projected strategy, $s,s' \in \states$, 
and $m,m' \in M$.
If $(s',m')$ is reachable from $(s,m)$ in $\game\restr_{\straa}$, then for all
beliefs $Y \subseteq \states$ with $s \in Y$ there exists a belief 
$Y' \subseteq \states$ with $s' \in Y'$ such that the state 
$(s',Y',\win_\straa(m'),\rec_\straa(m'), \actions_\straa(m'))$ is reachable 
from $(s,Y,\win_\straa(m),\rec_\straa(m), \actions_\straa(m))$ in $\game\restr_{\straa'}$
\end{lemma}
\begin{proof}
We will start the proof with one step reachability first. Assume there is an edge $(s,m) \rightarrow (s',m')$ in the Markov chain $\game\restr_\straa$ and the action labeling the transition in the POMDP is $a$, i.e., (i)~$s' \in \supp(\trans(s,a))$; (ii)~$a \in \supp(\straa_n(m))$; and (iii) $m' \in \supp(\straa_u(m,\obsmap(s'),a))$.
Let $(s,Y,\win_\straa(m),\rec_\straa(m), \actions_\straa(m))$ be a state in $\game\restr_{\straa'}$.
It follows that $a \in \actions_\straa(m)$ and there exists $Y' \subseteq \states$ and an edge 
$(Y,\win_\straa(m),\rec_\straa(m), \actions_\straa(m)) \stackrel{a}{\rightarrow} (Y',\win_\straa(m'),\rec_\straa(m'), \actions_\straa(m'))$ 
in the \projection graph $\prgr(\straa)$. Therefore, action $a$ is played with positive probability by the strategy $\straa'$ and with positive probability the memory is updated to $(Y', \win_\straa(m'),\rec_\straa(m'), \actions_\straa(m'))$. Therefore there is an edge $(s,Y,\win_\straa(m),\rec_\straa(m), \actions_\straa(m)) \rightarrow (s',Y',\win_\straa(m'),\rec_\straa(m'), \actions_\straa(m'))$ in the Markov chain $\game\restr_{\straa'}$.

We finish the proof by extending the one-step reachability to general reachability. If $(s',m')$ is reachable from $(s,m)$ in $\game\restr_\straa$ then there exists a finite path. Applying the one step reachability for every transition in the path gives us the desired result.
	\end{proof}

\smallskip\noindent\textbf{Random variable notation.}
For all $n \geq 0$ we write $X_n,Y_n,W_n,R_n,A_n,L_n$ for the random variables that correspond to the projection of the $n^{th}$ state of the Markov chain $\game\restr_{\straa'}$ on the $\states$ component, the belief $\powset(\states)$ component, the $\win_\straa$ component, the $\rec_\straa$ component, the $\actions_\straa$ component, and the $n^{th}$ action, respectively.

\smallskip\noindent\textbf{Run of the Markov chain $\game\restr_{\straa'}$.} 
A \emph{run} on the Markov chain $\game \restr_{\straa'}$ is an infinite sequence 
\[
(X_0,Y_0,W_0,R_0,A_0)\stackrel{L_0}{\rightarrow}(X_1,Y_1,W_1,R_1,A_1)\stackrel{L_1}{\rightarrow} \cdots 
\]
such that each finite prefix of the run is generated with positive probability on the Markov chain, 
i.e., for all $i \geq 0$, we have  
(i)~$L_i \in \supp(\straa'_n(Y_i,W_i,R_i,A_i))$; 
(ii)~$X_{i+1} \in \supp(\trans(X_i,L_i))$; and
(iii)~$(Y_{i+1},W_{i+1},R_{i+1},A_{i+1}) \in 
\supp(\straa'_u((Y_i,W_i,R_i,A_i),\obsmap(X_{i+1}),L_i))$. 
In the following lemma we establish important properties of the Markov chain 
$\game\restr_{\straa'}$ that are essential for our proof.

\begin{lemma}
\label{lem:prop}
Let $(X_0,Y_0,W_0,R_0,A_0) \stackrel{L_0}{\rightarrow} (X_1,Y_1,W_1,R_1,A_1) \stackrel{L_1}{\rightarrow} \cdots$ 
be a run of the Markov chain $\game\restr_{\straa'}$, then the following assertions hold for all $i \geq 0$:
\begin{enumerate}
\item $X_{i+1} \in \supp(\trans(X_i,L_i)) \cap Y_{i+1}$;
\item $(Y_i,W_i,R_i,A_i)\stackrel{L_i}{\rightarrow}(Y_{i+1},W_{i+1},R_{i+1},A_{i+1})$ is an edge in the \projection graph $\prgr(\straa)$;
\item if $W_i(X_i) = 1$, then $W_{i+1}(X_{i+1})=1$;
\item if $R_i(X_i) = 1$, then $R_{i+1}(X_{i+1})=1$; and
\item if $W_i(X_i) = 1$ and $R_i(X_i)=1$, then $\reward(X_i,L_i) = 1$.
\end{enumerate}
\end{lemma}
\begin{proof}
We prove all the points below:
\begin{enumerate}
\item The first point follows directly from the definition of the Markov chain and the \projected strategy $\straa'$.

\item The second point follows from the definition of the \projected strategy $\straa'$.

\item The third point follows from the first two points of this lemma and the first point of Lemma~\ref{lem:fix}.

\item The fourth point follows from the first two points of this lemma and the second point of Lemma~\ref{lem:fix}.

\item For the fifth point consider that $W_i(X_i) = 1$ and $R_i(X_i)=1$. 
Then there exists a memory $m \in M$ such that (i)~$\win_\straa(m) = W_i$, and 
(ii)~$\rec_{\straa}(m)=R_i$. 
Moreover, the state $(X_i,m)$ is a recurrent (since $R_i(X_i)=1$) and 
almost-sure winning state (since $W_i(X_i)=1$) in the Markov chain 
$\game\restr_\straa$. 
As $L_i \in \actions_\straa(m)$ it follows that $L_i \in \supp(\straa_n(m))$, 
i.e., the action $L_i$ is played with positive probability in state $X_i$ given
memory $m$, and $(X_i,m)$ is in an almost-sure winning recurrent class. 
By Lemma~\ref{lem:rec_class} it follows that the reward $\reward(X_i,L_i)$ must 
be $1$.
\end{enumerate}
The desired result follows.
\end{proof}

We now introduce the final notion of a \pseudo-recurrent state that is 
required to complete the proof. 
A state $(X,Y,W,R,A)$ of the Markov chain $\game\restr_{\straa'}$ is 
\pseudo-recurrent, if for all memory elements $m \in M$ that were merged to 
the memory element $(Y,W,R,A)$, the state $(X,m)$ of the Markov chain 
$\game\restr_{\straa}$ is recurrent. 
It will turn out that every recurrent state of the Markov chain 
$\game\restr_{\straa}$ is also \pseudo-recurrent.

\begin{definition}
A state $(X,Y,W,R,A)$ of the Markov chain $\game\restr_{\straa'}$ is called \pseudo-recurrent iff $R(X)=1$.
\end{definition}
Note that due to point~4 of Lemma~\ref{lem:prop} all the states reachable from a \pseudo-recurrent state are also \pseudo-recurrent.
In the following lemma we show that the set of \pseudo-recurrent states is reached with probability $1$.

\begin{lemma}
\label{lem:pseudo_reach}
With probability 1 a run of the Markov chain $\game\restr_{\straa'}$ reaches a \pseudo-recurrent state.
\end{lemma}
\begin{proof}
We show that from every state $(X,Y,W,R,A)$ in the Markov chain $\game\restr_{\straa'}$ there exists a reachable \pseudo-recurrent state. Consider a memory element $m \in M$ such that (i)~$\win_\straa(m) = W$, (ii)~$\rec_\straa(m) = R$, and (iii)~$\actions_\straa(m) = A$. Consider a state $(s,m)$ in the Markov chain $\game\restr_\straa$. By Property~1~(b) of Markov chains from every state in a Markov chain the set of recurrent states is reached with probability $1$, in particular there exists a reachable recurrent state $(X',m')$, such that $\rec_\straa(m')(X')=1$.
By Lemma~\ref{lem:reachability}, there exists $Y' \subseteq \states$ such that the state $(X',Y', \win_\straa(m), \rec_\straa(m'),\actions_\straa(m'))$ is reachable from $(X,Y,W,R,A)$ in the Markov chain $\game\restr_{\straa'}$, and moreover the state is \pseudo-recurrent.
As this is true for every state, we have that there is a positive probability of reaching a \pseudo-recurrent from every state in $\game\restr_{\straa'}$. This ensures by Property~1~(a) of Markov chains that \pseudo-recurrent states are reached with probability $1$.
\end{proof}

\begin{lemma}
\label{lem:proj}
The \projected strategy $\straa'$ is a finite-memory almost-sure winning strategy for the $\limavgone$
objective on the POMDP $\game$ with the reward function $\reward$.
\end{lemma}
\begin{proof}
The initial state of the Markov chain $\game\restr_{\straa'}$ is $(\{s_0\},\win_\straa(m_0), \rec_\straa(m_0), \actions_\straa(m_0))$ and as the strategy $\straa$ is an almost-sure winning strategy we have that $\win_\straa(m_0)(s_0) = 1$. It follows from the third point of Lemma~\ref{lem:prop} that every reachable state $(X,Y,W,R,A)$ in the Markov chain $\game\restr_{\straa'}$ satisfies that $W(X)=1$.

From every state a \pseudo-recurrent state is reached with probability $1$. It follows that all the recurrent states in the Markov chain $\game\restr_{\straa'}$ are also \pseudo-recurrent states. As in all reachable states $(X,Y,W,R,A)$ we have $W(X)=1$, by the fifth point of Lemma~\ref{lem:prop} it follows that every action $L$ played in a \pseudo-recurrent state $(X,Y,W,R,A)$ satisfies that the reward $\reward(X,L)=1$.  As this true for every reachable recurrent class, the fact that the \projected strategy is an almost-sure winning strategy for $\limavgone$ objective follows from Lemma~\ref{lem:rec_class}.
\end{proof}

\begin{theorem}[Strategy complexity]
\label{thm:stra}
The following assertions hold:
(1)~If there exists a finite-memory almost-sure winning strategy in the POMDP $\game = (\states,\act,\trans,\obs,\obsmap,s_0)$ with reward
function $\reward$ for the $\limavgone$ objective, then there exists a finite-memory almost-sure winning strategy with memory size at most
$2^{3 \cdot |\states|+|\act|}$.
(2)~Finite-memory almost-sure winning strategies for $\limavgone$ objectives in POMDPs 
in general require exponential memory and belief-based strategies are not sufficient.
\end{theorem}
\begin{proof}
The first item follows from Lemma~\ref{lem:proj} and the fact that the size of the 
memory set of the \projected strategy $\straa'$ of any finite-memory strategy 
$\straa$ (which is the size of the vertex set of the \projection graph 
of $\straa$) is bounded by $2^{3 \cdot |\states|+|\act|}$.
The second item is obtained as follows: (i)~the exponential memory 
requirement follows the almost-sure winning problem for PFA as 
the shortest witness to the non-emptiness problem for universal 
automata is exponential; and (ii)~the fact that belief-based strategies
are not sufficient follows from Example~\ref{ex:belief-suff}. 
\end{proof}

\newcommand{\wcs}{\mathit{wcs}}

\subsection{Computational complexity}
We will present an exponential time algorithm for the almost-sure winning 
problem in POMDPs with $\limavgone$ objectives under finite-memory strategies.
A naive double-exponential algorithm would be to enumerate all finite-memory 
strategies with memory bounded by $2^{3 \cdot |\states|+|\act|}$ (by Theorem~\ref{thm:stra}).
Our improved algorithm consists of two steps: (i)~first it constructs a special type of 
a \emph{belief-observation} POMDP; and we show that there exists a finite-memory almost-sure winning strategy 
for the objective $\limavgone$ iff there exists a randomized 
memoryless almost-sure winning strategy in the belief-observation POMDP for the $\limavgone$ objective; and
(ii)~ then we show how to determine whether there exists a randomized memoryless almost-sure winning strategy 
 in the belief-observation POMDP for the $\limavgone$ objective in polynomial time with respect to the size of 
the belief-observation POMDP. 
Intuitively a belief-observation POMDP satisfies that the current 
belief is always the set of states with current observation.

\begin{definition}
A POMDP $\game = (\states, \act, \trans, \obs, \obsmap, s_0)$ is a \emph{belief-observation} POMDP iff for every finite prefix $w= (s_0,a_0,s_1,a_1, \ldots, a_{n-1},s_n)$ the belief associated with the observation sequence $\rho = \obsmap(w)$ is the set of states with the last observation $\obsmap(s_n)$ of the observation sequence $\rho$, i.e., $\belief(\rho) = \obsmap^{-1}(\obsmap(s_n))$.
\end{definition}

\subsubsection{Construction of the belief-observation POMDP}
Given a POMDP $\game$ with the $\limavgone$ objective specified by a reward function $\reward$ 
we construct a belief-observation POMDP $\wb{\game}$ with the $\limavgone$ objective 
specified by a reward function $\wb{\reward}$, such that there exists a finite-memory 
almost-sure winning strategy ensuring the objective in the POMDP $\game$ 
iff there exists a randomized memoryless almost-sure winning strategy in the POMDP $\wb{\game}$. We will refer to this construction as $\wb{\game} = \red(\game)$.
Intuitively, the construction will proceed as follows: if there exists an almost-sure winning finite-memory strategy, then there exists an almost-sure winning strategy with memory bounded by $2^{3 \cdot |\states|+|\act|}$. 
This allows us to consider the memory elements $M= 2^{S} \times \{0,1\}^{|S|} \times \{0,1\}^{|S|} \times 2^{\act}$; and intuitively construct the product of the memory $M$ with the POMDP $\game$. 
First we define the set of winning memories $M_{\Win} \subseteq M$ defined as follows:
$$ M_{\Win} = \{ (Y,W,R,A) \mid \text{ for all } s \in Y \text{ we have } W(s)=1 \}$$
Note that every reachable memory $m$ of the \projected strategy of any finite-memory almost-sure winning strategy $\straa$  in the POMDP $\game$ must belong to $M_{\Win}$.


We proceed with the formal construction: 
let $\game = (\states, \act, \trans, \obs, \obsmap, s_0)$ be a POMDP with a reward function $\reward$. 
We construct a POMDP  $\wb{\game} = (\wb{\states}, \wb{\act}, \wb{\trans}, \wb{\obs}, \wb{\obsmap}, \wb{s}_0)$ with a reward function $\wb{\reward}$ as follows:
\begin{itemize}
\item The set of states $\wb{\states} = \wb{\states}_a \cup \wb{\states}_m \cup \{\wb{s}_0, \wb{s}_l \}$, consists of the action selection states $\wb{\states}_a = \states \times M$; the memory selection states $\wb{\states}_m = \states \times 2^\states \times \act \times M$; an initial state $\wb{s}_0$ and an absorbing loosing state $\wb{s}_l$.

\item The actions $\wb{\act}$ are $\act \cup M$, i.e, the actions $\act$ from the POMDP $\game$ to simulate action playing and the memory elements $M$ to simulate memory updates.

\item The observation set is $\wb{\obs} = (M) \cup (2^\states \times \act \times M) \cup \{\wb{s}_0\} \cup \{\wb{s}_l\}$, 
intuitively the first component of the state cannot be observed, i.e., the memory part of the state remains visible; and 
the two newly added states do have their own observations.

\item The observation mapping is then defined $\wb{\obsmap}((s,m)) =m, \: \wb{\obsmap}((s,Y,a,m)) = (Y,a,m), \: \wb{\obsmap}(\wb{s}_0) = \wb{s}_0$ and $\wb{\obsmap}(\wb{s}_l) = \wb{s}_l$.

\item As the precise probabilities do not matter for computing almost-sure winning states under finite-memory strategies for the $\limavgone$ objective we specify the transition function only as edges of the POMDP graph and the probabilities are uniform over the support set. We  define the transition function $\wb{\trans}$ in the following steps:

\begin{enumerate}
\item For every memory element $m \in M_{\Win} \cap \{(\{s_0\},W,R,A) \mid (\{s_0\},W,R,A)  \in M \}$ we add an edge $\wb{s}_0 \stackrel{m}{\rightarrow} (s_0,m)$. Intuitively the set from which $m$ is chosen contains all the possible initial memories of an almost-sure winning \projected strategy.

\item We will say an action $a \in \act$ is \emph{enabled} in the observation $(Y,W,R,A)$ iff (i)~$a \in A$, and (ii)~for all $\wh{s} \in Y$ we have if $W(\wh{s})=1$ and $R(\wh{s})=1$, then $\reward(\wh{s},a)=1$. 
Intuitively, a \projected strategy would only play enabled actions, therefore every action $a\in\act$ possibly played by a \projected strategy remains available in the state of the POMDP. This fact follows from 
the fifth point of Lemma~\ref{lem:prop}. 
For an action $a \in \act$ that is enabled in observation $(Y,W,R,A)$ we have an edge 
$(s,(Y,W,R,A)) \stackrel{a}{\rightarrow} (s',Y',a,(Y,W,R,A))$ iff both of the 
following conditions are satisfied: (i)~$s' \in \supp(\trans(s,a))$; and 
(ii)~$Y' = \bigcup_{\wh{s} \in Y} \supp(\trans(\wh{s},a)) \cap \obsmap^{-1}(\obsmap(s'))$, i.e., the belief update from belief $Y$, under action $a$ and observation $\obsmap(s')$.
For an action $a \in \act$ that is not enabled, 
we have an edge to the loosing absorbing state $\wb{s}_l$, i.e.,  
$(s,(Y,W,R,A)) \stackrel{a}{\rightarrow}\wb{s}_l$.

\item We will say an action $(Y',W',R',A') \in M$ is \emph{enabled} in the observation $(Y',a,(Y,W,R,A))$ iff 
(i)~for all states $\wh{s} \in Y$, if $W(\wh{s})=1$, then for all $\wh{s}' \in \supp(\trans(\wh{s},a)) \cap Y'$ 
we have $W'(\wh{s}')=1$, and 
(ii)~for all states $\wh{s} \in Y$, if $R(\wh{s})=1$, then for all $\wh{s}' \in \supp(\trans(\wh{s},a)) \cap Y'$ we have $R'(\wh{s}')=1$.
Observe that by Lemma~\ref{lem:fix}, the memory updates in a \projected strategy correspond to enabled actions. 
For an action $m$ that is enabled in the observation $(Y',a,(Y,W,R,A))$ we have an edge $(s',Y',a,(Y,W,R,A)) \stackrel{m}{\rightarrow} (s',m)$;
otherwise we have an edge to the loosing absorbing state $\wb{s}_l$, i.e., edge 
$(s',Y',a,(Y,W,R,A)) \stackrel{m}{\rightarrow} \wb{s}_l$ if $m$ is not enabled.

\item The state $\wb{s}_l$ is an absorbing state, i.e., for all $\wb{a} \in \wb{\act}$ we have that $\wb{s}_l \stackrel{\wb{a}}{\rightarrow} \wb{s}_l$.
\end{enumerate}
\end{itemize}
The reward function $\wb{\reward}$ is defined using the reward function $\reward$, i.e., 
$\wb{\reward}((s,(Y,W,R,A)),a) = \reward(s,a)$, and $\wb{\reward}((s',Y',a,(Y,W,R,A)),\wb{a}) =1 $ for all $\wb{a} \in \wb{\act}$.  
The reward for the initial state may be set to an arbitrary value, as the initial state is visited only once. 
The rewards in the absorbing state $\wb{s}_l$ are $0$ for all actions.

In the POMDP $\wb{\game}$ the belief is already included in the state space itself of the POMDP, 
and the belief represents exactly the set of states in which the POMDP can be with positive probability.
Hence we have the following lemma.

\begin{lemma}
The POMDP $\wb{\game}$ is a belief-observation PODMP.
\end{lemma}

We now show that the existence of a finite-memory almost-sure winning strategy for the objective $\limavgone$ 
in POMDP $\game$ implies the existence of a randomized memoryless almost-sure winning strategy in the POMDP $\wb{\game}$ 
and vice versa. 

\begin{lemma}
\label{lem:fmtonm}
If there exists a finite-memory almost-sure winning strategy for the $\limavgone$ objective with the reward function $\reward$ in the POMDP $\game$, 
then there exists a memoryless almost-sure winning strategy for the $\limavgone$ objective with the reward function $\wb{\reward}$ in the POMDP $\wb{\game}$.
\end{lemma}

	\begin{proof}
	Assume there  is a finite-memory almost-sure winning strategy $\straa$ in the POMDP $\game$, then the \projected strategy $\straa' = (\straa'_u,\straa'_n, M, (\{s_0\},W,R,A)) = \prst(\straa)$ is also an almost-sure winning strategy in the POMDP $\game$. We define the memoryless strategy $\wb{\straa}: \wb{\obs} \rightarrow \distr(\wb{\act})$ as follows:
	\begin{itemize}
	\item In the initial observation $\{\wb{s}_0\}$ play the action $(\{s_0\},W,R,A)$.
	\item In observation $(Y,W,R,A)$ play $\wb{\straa}((Y,W,R,A)) = \straa'_n((Y,W,R,A))$.
	\item In observation $(Y',a,(Y,W,R,A))$ we denote by $o$ the unique observation such that $Y' \subseteq \obsmap^{-1}(o)$, then $\wb{\straa}(Y',a,(Y,W,R,A)) = \straa'_u((Y,W,R,A),o,a)$.
	\item In the observation $\{\wb{s}_l\}$ no matter what is played the loosing absorbing state is not left. 
	\end{itemize}
	Let $G_1 = \game\restr_{\straa'}$ and $G_2 = \wb{\game} \restr_{\wb{\straa}}$. We will for simplicity collapse the edges in the Markov Chain $G_2$, i.e., we will consider an edge $(s,(Y,W,R,A)) \stackrel{a}{\rightarrow} (s',(Y',W',R',A'))$ whenever $(s,(Y,W,R,A)) \stackrel{a}{\rightarrow} (s',Y',a,(Y,W,R,A)) \stackrel{(Y',W',R',A')}{\rightarrow} (s',(Y',W',R',A'))$. Note that in the first step the Markov chain $G_2$ reaches a state $(s_0,(\{s_0\},W,R,A))$, and one can observe that the Markov chains reachable from the initial state in $\game_1$ and the state $(s_0,(\{s_0\},W,R,A))$ in $\game_2$ are isomorphic (when collapsed edge are considered). 
	By Lemma~\ref{lem:rec_class}, all the reachable recurrent classes in the Markov chain $G_1$ have all the rewards equal to $1$. As all the intermediate states in the Markov chain $G_2$ that were collapsed have reward according to $\wb{\reward}$ equal to  $1$ for all the actions, it follows that all the reachable recurrent classes in $G_2$ have all the rewards assigned by the reward function $\wb{\reward}$  equal to $1$. By Lemma~\ref{lem:rec_class} it follows that $\wb{\straa}$ is a memoryless almost-sure winning strategy in the POMDP $\wb{\game}$ for the objective $\limavgone$.
	\end{proof}

	\begin{lemma}
	\label{lem:nmtofm}
	If there exists a memoryless almost-sure winning strategy for the $\limavgone$ objective with the reward function $\wb{\reward}$ in the POMDP $\wb{\game}$, then there exists a finite-memory almost-sure winning strategy for the $\limavgone$ objective withe the reward function $\reward$ in the POMDP $\game$.
	\end{lemma}

	\begin{proof}
	Let $\wb{\straa}$ be a memoryless almost-sure winning strategy. Intuitively we use everything from the state except the first component as the memory in the constructed finite-memory strategy $\straa = (\straa_u,\straa_n, M, m_0)$, i.e., 
	\begin{itemize}
	\item $\straa_n((Y,W,R,A)) = \wb{\straa}((Y,W,R,A))$;
	\item $\straa_u((Y,W,R,A),o,a)$ we update uniformly to the elements from the set $\supp(\wb{\straa}((Y',a,(Y,W,R,A))))$, where $Y'$
is the belief update from $Y$ under observation $o$ and action $a$.
	\item $m_0 = \wb{\straa}(\{s'_0\})$, this can be in general a probability distribution, in our setting the initial memory is deterministic. However, this can modeled by adding an additional initial state from which the required memory update is going to be modeled.
	\end{itemize}
	For simplicity we collapse edges in the Markov chain $\game_2 = \wb{\game}\restr_{\wb{\straa}}$ and write $(s,(Y,W,R,A)) \stackrel{a}{\rightarrow} (s',(Y',W',R',A'))$ whenever $(s,(Y,W,R,A)) \stackrel{a}{\rightarrow} (s',Y',a,(Y,W,R,A)) \stackrel{(Y',W',R',A')}{\rightarrow} (s',(Y',W',R',A'))$. One can observe that the graphs reachable from the state $(s_0,(\{s_0\},W,R,A))$ in the Markov chain $\game_1 =\game \restr_{\straa}$ and the state 
	$(s_0,(\{s_0\},W,R,A))$ in the Markov chain $\game_2$ are isomorphic when the collapsed edges are considered. As the strategy $\wb{\straa}$ is almost-sure winning, it follows that the loosing absorbing state is not reachable in $\game_2$, 
	and by Lemma~\ref{lem:rec_class}, in every reachable recurrent class all the rewards are $1$. It follows that in every reachable recurrent class in $\game_1$ all the rewards are $1$ and the desired result follows.
	\end{proof}\

\subsubsection{Polynomial time algorithm for belief-observation POMDPs}
We will  present a polynomial-time algorithm to determine the set of states from 
which there exists a memoryless almost-sure winning strategy for the objective $\limavgone$ 
in the belief-observation POMDP $\wb{\game} =(\wb{\states},\wb{\act},\wb{\trans},\wb{\obs},\wb{\obsmap},\wb{s}_0)$.

For simplicity in presentation we enhance the belief-observation POMDP $\wb{\game}$ with an additional function 
$\wb{\av} : \wb{\obs} \rightarrow \powset(\wb{\act}) \setminus \emptyset$ denoting the set of actions that are 
available for an observation, i.e., in this part we will consider POMDPs to be a tuple 
$\red(\game)  = \wb{\game} =(\wb{\states},\wb{\act},\wb{\trans},\wb{\obs},\wb{\obsmap},\wb{\av},\wb{s}_0)$. Note that this does not increase the 
expressive power of the model, as an action not available in an observation may be simulated 
by an edge leading to a loosing absorbing state.

\smallskip\noindent{\bf Almost-sure winning observations.}
Given a POMDP $\wb{\game} =(\wb{\states},\wb{\act},\wb{\trans},\wb{\obs},\wb{\obsmap},\wb{\av},\wb{s}_0)$ and an objective $\psi$, let $\almostm(\psi)$ 
denote the set of observations $\wb{o} \in \wb{\obs}$, such that there exists a memoryless almost-sure winning strategy ensuring 
the objective from every state $\wb{s} \in \wb{\obsmap}^{-1}(\wb{o})$, i.e.,
$ \almostm(\psi) = \{ \wb{o} \in \wb{\obs} \mid \text{ there exists a memoryless strategy $\wb{\straa}$, such that for all } \wb{s} \in \wb{\obsmap}^{-1}(\wb{o}) \text{ we have } \prb_{\wb{s}}^{\wb{\straa}}(\psi)=1\}$. 
Our goal is to compute the set $\almostm(\limavgone)$ given the belief-observation POMDP $\wb{\game} = \red(\game)$.
Our algorithm will reduce the computation to \emph{safety} and \emph{reachability} objectives
which are defined as follows:
\begin{itemize}
 	\item \emph{Reachability and safety objectives.}
	Given a set $\target \subseteq \wb{S}$ of target states, the \emph{reachability} objective 
	$\Reach(\target) = \{ (\wb{s}_0, \wb{a}_0, \wb{s}_1, \wb{a}_1, \wb{s}_2 \ldots) \in \Omega \mid \exists k \geq 0:  \wb{s}_k \in \target\}$
	requires that a target state in $\target$ is visited at least once.
	Dually, the \emph{safety} objective $\safe(\targetsafe) = \{ (\wb{s}_0, \wb{a}_0, \wb{s}_1, \wb{a}_1, \wb{s}_2 \ldots) \in \Omega \mid \forall k \geq 0:  \wb{s}_k \in \targetsafe\}$ 
	requires that only states in $\targetsafe$ are visited.


\end{itemize}

In the first step of the computation of the winning observations, 
we will restrict the set of available actions in the belief-observation POMDP $\wb{\game}$. 
Observe, that any almost-sure winning strategy must avoid reaching the loosing absorbing state $\wb{s}_l$. 
Therefore, we restrict the actions in the POMDP to only so called \emph{allowable or safe}  actions.

\begin{itemize}
\item \emph{(Allow).} 
Given a set of observations $\wb{O} \subseteq \wb{\obs}$ and an observation $\wb{o} \in \wb{O}$ we define by $\allow(\wb{o},\wb{O})$ the set of actions that when played in observation $\wb{o}$ (in any state $\wb{\obsmap}^{-1}(\wb{o})$) ensure that the next observation is in $\wb{O}$, i.e.,
	$$\allow(\wb{o},\wb{O}) = \{ \wb{a} \in \wb{\av}(\wb{o}) \mid \bigcup_{\wb{s} \in \wb{\obsmap}^{-1}(\wb{o})} \wb{\obsmap}(\supp(\wb{\trans}(\wb{s},\wb{a}))) \subseteq \wb{O} \}$$
	\end{itemize}

	We will denote by $\wb{S}_{\good} = \wb{\states} \setminus \wb{s}_l$ the set of states of the POMDP without the loosing absorbing state. We compute the set of observations $\almostm(\safe(\wb{S}_{\good}))$ from which there exists a memoryless strategy ensuring almost-surely that the loosing absorbing state $\wb{s}_l$ is not visited. 

	 We construct a restricted POMDP $\wt{G} = (\wt{\states},\wt{\act},\wt{\trans},\wt{\obs},\wt{\obsmap},\wt{\av},\wt{s}_0)$, where $\wt{\act} = \wb{\act}$, $\wt{s}_0 =\wb{s}_0$ and the rest is defined as follows:
	\begin{itemize}
	\item the set of states is restricted to $\wt{\states} = \wb{\obsmap}^{-1}(\almostm(\safe(\wb{S}_{\good})))$, 
	\item the set of \emph{safe} actions available in observation $\wt{o}$ are 
	$$\wt{\av}(\wt{o}) = \wb{\av}(\wt{o}) \cap \allow(\wt{o}, \almostm(\safe(\wb{S}_{\good})))$$
	\item the transition function $\wt{\trans}$ is defined as $\wb{\trans}$ but restricted to states $\wt{\states}$, and 
	\item the observation $\wt{\obs} = \almostm(\safe(\wb{S}_{\good}))$ and the observation mapping $\wt{\obsmap}$ is defined as $\wb{\obsmap}$ restricted to states $\wt{\states}$.

	\end{itemize}


	Note that any memoryless almost-sure winning strategy $\wb{\straa}$ for the objective $\limavgone$ on the POMDP $\wb{\game}$ can be interpreted on the restricted POMDP $\wt{\game}$ as playing an action from the set $\wb{\av}(o)\setminus\wt{\av}(o)$ or reaching a state in $\wb{\states} \setminus \wt{\states}$ leads to a contradiction to the fact that $\wb{\straa}$ is an almost-sure winning strategy (the loosing absorbing state is reached with positive probability). Similarly due to the fact that the POMDP $\wb{\game}$ is a belief-observation POMDP it follows that the restricted POMDP $\wt{\game}$ is also a belief-observation POMDP.

	We define a subset of states of the belief-observation POMDP $\wt{\game}$ that intuitively correspond to winning \pseudo-recurrent states (wcs), i.e.,
	$\wt{S}_{\wcs} = \{(s,(Y,W,R,A)) \mid W(s) = 1, R(s) = 1 \}$. Finally we compute the set of observations $\wt{\w} = \almostm(\reach(S_{\wcs}))$ in the restricted belief-observation POMDP $\wt{\game}$. We show that the set of observations $\wt{\w}$ is equal to the set of observations $\almostm(\limavgone)$ in the POMDP $\wb{\game}$. In the following two lemmas we establish the required inclusions:

	\begin{lemma} 
	$\wt{\w} \subseteq \almostm(\limavgone)$.
	\end{lemma}
	\begin{proof}
	Let $\wt{\straa}$ be a memoryless almost-sure winning strategy for the objective $\reach(\wt{S}_{\wcs})$ from every state that has  observation $\wt{o}$ in the POMDP $\wt{\game}$, for all $\wt{o}$ in $\wt{\w}$. We will show that the strategy $\wt{\straa}$ also almost-surely ensures the $\limavgone$ objective in POMDP $\wb{\game}$. 

	Consider the Markov chain $\wb{\game}\restr_{\wt{\straa}}$. As the strategy $\wt{\straa}$ plays in the POMDP $\wt{\game}$ that is restricted only to actions that keep the game in observations $\almostm(\safe(\wb{S}_{\good}))$ it follows that  
  the loosing absorbing state $\wb{s}_l$ is not reachable in the Markov chain $\wb{\game}\restr_{\wt{\straa}}$. Also the strategy ensures that the set of states $\wt{S}_{\wcs}$ is reached with probability $1$ in the Markov chain $\wb{\game}\restr_{\wt{\straa}}$. 

	 Let $(s,(Y,W,R,A)) \stackrel{a}{\rightarrow} (s',Y',a,(Y,W,R,A))$ be an edge in the Markov chain $\wb{\game}\restr_{\wt{\straa}}$ and assume that $(s,(Y,W,R,A)) \in \wt{S}_{\wcs}$, i.e., $W(s) = 1$ and $R(s)=1$. The only actions $(Y',W',R',A')$ available in the state $(s',Y',a,(Y,W,R,A))$ are \emph{enabled} actions that satisfy:

	 \begin{itemize}
	 \item for all $\wh{s} \in Y$, if $W(\wh{s})=1$, then for all $\wh{s}' \in \supp(\trans(\wh{s},a)) \cap Y'$  we have $W'(\wh{s}')=1$, and
	 \item for all $\wh{s} \in Y$, if $R(\wh{s})=1$, then for all $\wh{s}' \in \supp(\trans(\wh{s},a)) \cap Y'$  we have $R'(\wh{s}')=1$.
	 \end{itemize}
 
	 It follows that all the states reachable in one step from $(s',Y',a,(Y,W,R,A))$ are also in $\wt{S}_{\wcs}$. Similarly, for every \emph{enabled} action in state $(s,(Y,W,R,A))$ we have that for all $s \in Y$, if $W(s)=1$ and $R(s)=1$, then $\reward(s,a)=1$. Therefore, all the rewards $\wb{\reward}$ from states in $\wt{S}_{\wcs}$ are $1$, and all the intermediate states $(s',Y',a,(Y,W,R,A))$ have reward $1$ by definition. This all together ensures that after reaching the set $\wt{S}_{\wcs}$ only rewards $1$ are received, and as the set $\wt{S}_{\wcs}$ is reached with probability $1$, it follows by Lemma~\ref{lem:rec_class} that $\wt{\straa}$ is an almost-sure winning strategy in the POMDP $\wb{\game}$ for the $\limavgone$ objective.
	\end{proof}

	\begin{lemma} 
	$ \almostm(\limavgone) \subseteq \wt{\w}$.
	\end{lemma}
	\begin{proof}
	Assume towards contradiction that there exists an observation $\wt{o} \in \almostm(\limavgone) \setminus \wt{\w}$. 
	As we argued before, the observation $\wt{o}$ must belong to the set $\almostm(\safe(\wb{S}_{\good}))$, otherwise with positive probability the loosing absorbing state $\wb{s}_l$ is reached.
	
	As $\wt{o} \in \almostm(\limavgone)$ there exists a memoryless almost-sure winning strategy $\wb{\straa}$ in the POMDP $\wb{\game}$ for the objective $\limavgone$. By Lemma~\ref{lem:nmtofm}  there exists a finite-memory almost-sure winning strategy $\straa$ in POMDP $\game$ for the $\limavgone$ objective and by Theorem~\ref{thm:stra} there exists a finite-memory almost-sure winning \projected strategy $\straa' = \prst(\straa)$ in POMDP $\game$.
 
 Let $(s,(Y,W,R,A))$ be a reachable \pseudo-recurrent state in the Markov chain $\game\restr_{\straa}$. Then by the definition of \pseudo-recurrent states we have that $R(s) = 1$, and as the strategy $\straa'$ is almost-sure winning we also have that $W(s) = 1$, i.e., if we consider the state $(s,(Y,W,R,A))$ of the POMDP $\wb{\game}$ we obtain that the state belongs to the set $\wt{S}_{\wcs}$.
 Note that by Lemma~\ref{lem:pseudo_reach} the set of \pseudo-recurrent states in the Markov chain $\game\restr_{\straa'}$ is reached with probability $1$. 
	  By the construction presented in Lemma~\ref{lem:fmtonm} we obtain a memoryless almost-sure winning strategy $\wb{\straa}$ for the objective $\limavgone$ in the POMDP $\wb{\game}$. Moreover, we have that the Markov chains $\game\restr_{\straa'}$ and $\wb{\game}\restr_{\wb{\straa}}$ are isomorphic when simplified edges are considered. In particular it follows that the set of states $\wt{S}_{\wcs}$ is reached with probability $1$ in the Markov chain $\wb{\game}\restr_{\wb{\straa}}$ and therefore $\wb{\straa}$ is a witness strategy for the fact that the observation $\wt{o}$ belongs to the set $\wt{\w}$. The contradiction follows.
	\end{proof}

To complete the computation for almost-sure winning for $\limavgone$ objectives
we now present polynomial time solutions for almost-sure safety and 
almost-sure reachability objectives for randomized memoryless strategies
 in the belief-observation POMDP $\wb{\game}$.
We start with a few notations below:

\begin{itemize}
\item \emph{(Pre).} The predecessor function given a set of observations $\wb{\obsset}$ selects the observations $\wb{o} \in \wb{\obsset}$ such that $\allow(\wb{o},\wb{\obsset})$ is non-empty , i.e.,
$$\pre(\wb{\obsset}) = \{ \wb{o} \in \wb{\obs} \mid \allow(\wb{o},\wb{\obsset}) \not = \emptyset \}.$$

\item \emph{(Apre).} Given a set $\wb{Y} \subseteq \wb{\obsset}$ of observations  and a set
$\wb{X} \subseteq \wb{S}$ of states such that $\wb{X} \subseteq \wb{\obsmap}^{-1}(\wb{Y})$, the set $\apre(\wb{Y},\wb{X})$ denotes the states from $\wb{\obsmap}^{-1}(\wb{Y})$ such that there exists an action that ensures that the next observation is in $\wb{Y}$ and the set $\wb{X}$ is reached with positive probability, i.e.,:
$$ \apre(\wb{Y},\wb{X}) = \{ \wb{s} \in \wb{\obsmap}^{-1}(\wb{Y}) \mid \exists \wb{a} \in \allow(\wb{\obsmap}(\wb{s}),\wb{Y}) \text{ such that } \supp(\wb{\trans}(\wb{s},\wb{a})) \cap \wb{X} \not =  \emptyset \}.$$

\item \emph{(ObsCover).} For a set $\wb{U} \subseteq \wb{S}$ of states we define the $\obscover(\wb{U}) \subseteq \wb{\obs}$ to be the set of observations $\wb{o}$ such that all states
with observation $\wb{o}$ are in $\wb{U}$, 
i.e., $\obscover(\wb{U}) = \{ \wb{o} \in \wb{\obs} \mid \wb{\obsmap}^{-1}(\wb{o}) \subseteq \wb{U}\}$.
\end{itemize}
Using the above notations we present the solution of almost-sure 
winning for safety and reachability objectives.

\smallskip\noindent{\bf Almost-sure winning for safety objectives.} 
Given a safety objective $\safe(\wb{F})$, for a set $\wb{F} \subseteq \wb{S}$ of states, 
let $\obsset_{\wb{F}}=\obscover(\wb{F})$ denote the set of observations $\wb{o}$ such that 
$\wb{\obsmap}^{-1}(\wb{o}) \subseteq \wb{F}$, i.e., all states $\wb{s} \in \wb{\obsmap}^{-1}(\wb{o})$ 
belong to $\wb{F}$.
We denote by $\nu X$ the greatest fixpoint and by $\mu X$ the least fixpoint. 
Let 
\[
Y^* = \nu Y.(\obsset_{\wb{F}} \cap \pre(Y)) =\nu Y. (\obscover(\wb{F}) \cap \pre(Y))
\]
be the greatest fixpoint of the function $f(Y)= \obsset_{\wb{F}} \cap \pre(Y)$. 
Then the set $Y^*$ is obtained by the following computation:
\begin{enumerate}
\item $Y_{0} \leftarrow \obsset_{\wb{F}}$; and
\item repeat $Y_{i+1} \leftarrow \pre(Y_{i})$ until a fixpoint is reached.
\end{enumerate}
We show that $Y^*=\almostm(\safe(\wb{F}))$.

\begin{lemma}
\label{lem:nonemptyallow}
For every observation $\wb{o} \in Y^*$ we have $\allow(\wb{o},Y^*)\neq \emptyset$ 
(i.e., $\allow(\wb{o},Y^*)$ is non-empty).
\end{lemma}
\begin{proof}
Assume towards contradiction that there exists an observation $\wb{o} \in Y^*$ 
such that $\allow(\wb{o},Y^*)$ is empty. 
Then $\wb{o} \not\in \pre(Y^*)$ and hence the observation 
must be removed in the next iteration of the algorithm. 
This implies $\pre(Y^*) \not = Y^*$, we reach a contradiction that 
$Y^*$ is a fixpoint.
\end{proof}

\begin{lemma}
The set $Y^*$ is the set of almost-sure winning observations for the
safety objective $\safe(\wb{F})$, i.e., $Y^* = \almostm(\safe(\wb{F}))$, 
and can be computed in linear time.
\end{lemma}

\begin{proof}
We prove the two desired inclusions:
(1)~$Y^* \subseteq \almostm(\safe(\wb{F}))$; and 
(2)~$\almostm(\safe(\wb{F})) \subseteq Y^*$.
\begin{enumerate}
\item \emph{(First inclusion).} By the definition of $Y_0$ we have that 
$\wb{\obsmap}^{-1}(Y_0) \subseteq \wb{F}$. 
As $Y_{i+1} \subseteq Y_{i}$ we have that $\wb{\obsmap}^{-1}(Y^*) \subseteq \wb{F}$. 
By Lemma~\ref{lem:nonemptyallow}, for all observations $\wb{o}\in Y^*$ we have 
$\allow(\wb{o},Y^*)$ is non-empty. 
A pure memoryless that plays some action from $\allow(\wb{o},Y^*)$ in $\wb{o}$, for 
$\wb{o}\in Y^*$, ensures that the next observation is in $Y^*$.
Thus the strategy ensures that only states from 
$\wb{\obsmap}^{-1}(Y^*) \subseteq \wb{F}$ are visited, 
and therefore is an almost-sure winning strategy for the safety objective.

\item \emph{(Second inclusion).} We prove that there is no almost-sure winning
strategy from $\wb{\obs} \setminus Y^*$ by induction:
\begin{itemize}
\item \textbf{(Base case).} There is no almost-sure winning strategy from observations $\wb{\obs} \setminus Y_0$. Note that $Y_0 = \obsset_{\wb{F}}$.
In every observation $\wb{o} \in \wb{\obs} \setminus Y_0$ there exists a state $\wb{s} \in \wb{\obsmap}^{-1}(\wb{o})$ such that $\wb{s} \not \in \wb{F}$. As $\game$ is a belief-observation POMDP there is a positive probability of being in state $\wb{s}$, and therefore not being in $\wb{F}$. 
\item \textbf{(Inductive step).} We show that there is no almost-sure winning strategy from observations in $\wb{\obs} \setminus Y_{i+1}$. Let $Y_{i+1} \not = Y_{i}$ and $\wb{o} \in Y_{i} \setminus Y_{i+1}$ (or equivalently $(\wb{\obs} \setminus Y_{i+1}) \setminus (\wb{\obs} \setminus Y_{i})$). As the observation $\wb{o}$ is removed from $Y_i$ it follows that $\allow(\wb{o}, Y_i) = \emptyset$. It follows that no matter what action is played, there is a positive probability of being in a state $\wb{s} \in \wb{\obsmap}^{-1}(\wb{o})$ such that playing the action would leave the set $\wb{\obsmap}^{-1}(Y_i)$ with positive probability, and thus reaching the observations $\wb{\obs} \setminus Y_i$ from which there is no almost-sure winning strategy by induction hypothesis.
\end{itemize}
\end{enumerate}
This shows that $Y^*=\almostm(\safe(\wb{F}))$, and the linear time computation follows
from the straight forward computation of greatest fixpoints.
The desired result follows.
\end{proof}


\smallskip\noindent{\bf Almost-sure winning for reachability objectives.} 
Consider a set $\wb{T}\subseteq \wb{S}$ of target states, and the reachability objective
$\Reach(\wb{T})$. 
We will show that: 
$$\almostm(\Reach(\wb{T})) = \nu Z. \obscover(\mu X. ((\wb{T} \cap \obsmap^{-1}(Z) ) \cup \apre(Z,X))).$$
Let $Z^* = \nu Z. \obscover(\mu X. ((\wb{T} \cap \obsmap^{-1}(Z) ) \cup \apre(Z,X)))$. 
In the following two lemmas we show the two desired inclusions, i.e., 
$\almostm(\Reach(\wb{T})) \subseteq Z^*$ and then we show that 
$Z^* \subseteq \almostm(\Reach(\wb{T}))$.

	\begin{lemma}
	\label{lem:bsuby}
	$\almostm(\Reach(\wb{T})) \subseteq Z^*$.
	\end{lemma}
	\begin{proof}
	Let $W^* = \almostm(\Reach(\wb{T}))$.
	We first show that $W^*$ is a fixpoint of the function 
	\[
	f(Z)= \obscover (\mu X. ((\wb{T} \cap \obsmap^{-1}(Z) ) \cup
	\apre(Z,X))),
	\] 
	i.e., 
	we will show that $W^*= \obscover(\mu X. ((\wb{T} \cap \obsmap^{-1}(W^*) ) \cup \apre(W^*,X)))$ . 
	As $Z^*$ is the greatest fixpoint it will follow that $W^* \subseteq Z^*$.
	Let 
	\[
	X^* = (\mu X.((\wb{T} \cap \obsmap^{-1}(W^*) ) \cup \apre(W^*,X))),
	\]
	and $\wh{X}^*=\obscover(X^*)$.
	Note that by definition we have $X^* \subseteq \wb{\obsmap}^{-1}(W^*)$ as the inner 
	fixpoint computation only computes states that belong to $\wb{\obsmap}^{-1}(W^*)$.
	Assume towards contradiction that $W^*$ is not a fixpoint, i.e., 
	$\wh{X}^*$ is a strict subset of $W^*$. 
	For all states $\wb{s}\in \wb{\obsmap}^{-1}(W^*) \setminus X^*$, 
	for all actions $\wb{a} \in \allow(\wb{\obsmap}(\wb{s}),W^*)$ we have 
	$\supp(\wb{\trans}(\wb{s},\wb{a})) \subseteq (\wb{\obsmap}^{-1}(W^*)\setminus X^*)$. 
	Consider any randomized memoryless almost-sure winning strategy $\straa^*$ 
	from $W^*$ and we consider two cases:
	\begin{enumerate}
	\item Suppose there is a state $\wb{s} \in \wb{\obsmap}^{-1}(W^*)\setminus X^*$ 
	such that an action that does not belong to $\allow(\wb{\obsmap}(\wb{s}),W^*)$ 
	is played with positive probability by $\straa^*$.
	Then with positive probability the observations from $W^*$ are left 
	(because from some state with same observation as $\wb{s}$ an observation in the
	complement of $W^*$ is reached with positive probability).
	Since from the complement of $W^*$ there is no randomized memoryless almost-sure winning 
	strategy (by definition), it contradicts that $\straa^*$ is an almost-sure winning 
	strategy from $W^*$.
	\item Otherwise for all states $\wb{s} \in \wb{\obsmap}^{-1}(W^*)\setminus X^*$ 
	the strategy $\straa^*$ plays only actions in $\allow(\wb{\obsmap}(\wb{s}),W^*)$, 
	and then the probability to reach $X^*$ is zero, i.e., $\safe(\wb{\obsmap}^{-1}(W^*)\setminus X^*)$ 
	is ensured. 
	Since all target states in $\wb{\obsmap}^{-1}(W^*)$ belong to $X^*$ 
	(they get included in iteration~0 of the fixpoint computation) 
	it follows that $(\wb{\obsmap}^{-1}(W^*)\setminus X^*) \cap \wb{T} =\emptyset$, and 
	hence $\safe(\wb{\obsmap}^{-1}(W^*) \setminus X^*)
	\cap \Reach(\wb{T})=\emptyset$,
	and we again reach a contradiction that $\straa^*$ is an almost-sure
	winning strategy. 
	\end{enumerate}
	It follows that $W^*$ is a fixpoint, and thus we get that $W^* \subseteq Z^*$.
	\end{proof}

	\begin{lemma}
	\label{lem:ysubb}
	$Z^* \subseteq \almostm(\Reach(\wb{T}))$.
	\end{lemma}
	\begin{proof}
	Since the goal is to reach the set $\wb{T}$, wlog we assume the set $\wb{T}$ to be absorbing.
	We define a randomized memoryless strategy $\straa^*$ for the objective $\almostm(\Reach(\wb{T}))$ 
	as follows: for an observation $\wb{o} \in Z^*$, play all actions from the set 
	$\allow(\wb{o}, Z^*)$ uniformly at random.
	Since the strategy $\straa^*$ plays only actions
	in $\allow(\wb{o}, Z^*)$, for $\wb{o} \in Z^*$, it ensures that the set of states 
	$\wb{\obsmap}^{-1}(Z^*)$ is not left, (i.e., $\safe(\wb{\obsmap}^{-1}(Z^*))$ is ensured).
	We now analyze the computation of the inner fixpoint, i.e., analyze the computation 
	of $\mu X.((\wb{T} \cap \wb{\obsmap}^{-1}(Z^*) ) \cup \apre(Z^*,X)))$
	as follows:
	\begin{itemize}
	\item $X_0 = 
	(\wb{T} \cap \wb{\obsmap}^{-1}(Z^*) ) \cup \apre(Z^*,\emptyset)))
	= \wb{T} \cap \wb{\obsmap}^{-1}(Z^*) \subseteq \wb{T}$ (since  $\apre(Z^*,\emptyset)$ is $\emptyset$);
	\item $X_{i+1} = (\wb{T} \cap \wb{\obsmap}^{-1}(Z^*)) \cup
	\apre(Z^*,X_i)))$
	\end{itemize}
	Note that we have $X_0 \subseteq \wb{T}$.
	For every state $\wb{s}_j \in X_j$ the set of played actions $\allow(\wb{\obsmap}(\wb{s}_j),
	Z^*)$ contains an action $\wb{a}$ such that $\supp(\wb{\trans}(\wb{s}_j,\wb{a})) \cap X_{j-1}$ is non-empty.
	Let $C$ be an arbitrary reachable recurrent class in the Markov chain $\game \restr
	\straa^*$ reachable from a state in $\wb{\obsmap}^{-1}(Z^*)$.
	Since  $\safe(\wb{\obsmap}^{-1}(Z^*))$ is ensured, it follows that $C \subseteq \wb{\obsmap}^{-1}(Z^*)$.
	Consider a state in $C$ that belongs to $X_j \setminus X_{j-1}$ for $j \geq 1$.
	Since the strategy ensures that for some action $\wb{a}$ played with positive 
	probability we must have $\supp(\wb{\trans}(\wb{s}_j,\wb{a})) \cap X_{j-1} \neq \emptyset$,
	it follows that $C \cap X_{j-1} \neq \emptyset$.
	Hence by induction $C \cap X_0 \neq \emptyset$.
	It follows $C \cap \wb{T} \neq \emptyset$.
	Hence all reachable recurrent classes $C$ that intersect with $Z^*$ are contained in $Z^*$, but 
	not contained in $Z^*\setminus \target$, i.e., all reachable recurrent classes are 
	the absorbing states in $\target$.
	Thus the strategy $\straa^*$ ensures that $\wb{T}$ is reached with probability~1.
	Thus we have $Z^* \subseteq \almostm(\Reach(\wb{T}))$. 
	\end{proof}

	\begin{lemma}
	The set $\almostm(\Reach(\wb{T}))$ can be computed 
	in quadratic time for belief-observation POMDPs, for target set $\wb{T} \subseteq \wb{S}$.
	\end{lemma}
	\begin{proof}
	Follows directly from Lemma~\ref{lem:bsuby} and Lemma~\ref{lem:ysubb}. 
	\end{proof}

	\smallskip\noindent{\bf The EXPTIME-completeness.}
	In this section we first showed that given a POMDP $\game$ with a $\limavgone$
	objective we can construct an exponential size belief-observation POMDP
	$\wb{\game}$ and the computation of the almost-sure winning set for 
	$\limavgone$ objectives is reduced to the computation of the almost-sure 
	winning set for safety and reachability objectives, for which we established
	linear and quadratic time algorithms respectively.
	This gives us an $2^{O(|S| + |\act|)}$ time algorithm to decide 
	(and construct if one exists) the existence of finite-memory almost-sure
	winning strategies in POMDPs with $\limavgone$ objectives.
	The EXPTIME-hardness for almost-sure winning easily follows from the result
	of Reif for two-player partial-observation games with safety objectives~\cite{Reif84}:
	(i)~First observe that in POMDPs, if almost-sure safety is violated, then it is
	violated in a finite prefix which has positive probability, and hence for 
	almost-sure safety, the probabilistic player can be treated as an adversary. 
	This shows that the almost-sure safety problem for POMDPs is EXPTIME-hard.
	(ii)~The almost-sure safety problem reduces to almost-sure winning for 
	limit-average objectives by assigning reward~1 to safe states, reward~0
	to non-safe states and make the non-safe states absorbing.
	It follows that POMDPs with almost-sure winning for $\limavgone$ objectives
	under finite-memory strategies is EXPTIME-hard.

	\begin{theorem}
	The following assertions hold:
	\begin{enumerate}
	\item Given a POMDP $G$ with $|S|$ states, $|\act|$ actions, and a $\limavgone$ objective, the existence (and the construction if one exists) of a 
	finite-memory almost-sure winning strategy can be 
	achieved in  $2^{O(|S|+|\act|)}$ time.

	\item The decision problem of given a POMDP and a $\limavgone$ objective whether
	there exists a finite-memory almost-sure winning strategy
	is EXPTIME-complete.
	\end{enumerate}

	\end{theorem}




\section{Finite-memory strategies with Quantitative Constraint}
We will show that the problem of deciding whether there exists a finite-memory 
(as well as an infinite-memory) almost-sure winning strategy for the objective 
$\limavghalf$ is undecidable. 
We present a reduction from the standard undecidable problem for probabilistic 
finite automata (PFA). 
A PFA $\automaton = (\states, \act, \trans,F,s_0)$ is a special case of a POMDP 
$\game = (\states,\act,\trans,\obs,\obsmap,s_0)$ with a single observation 
$\obs=\set{o}$ such that for all states $s\in \states$ we have $\obsmap(s)=o$.
Moreover, the PFA proceeds for only finitely many steps, and has a set $F$ of 
desired final states.
The \emph{strict emptiness problem} asks for the existence of a strategy $w$ 
(a finite word over the alphabet $\act$) such that the measure of the runs 
ending in the desired final states $F$ is strictly greater than $\frac{1}{2}$;
and the strict emptiness problem for PFA is undecidable~\cite{PazBook}.


\smallskip\noindent{\bf Reduction.}
Given a PFA $\automaton = (\states, \act, \trans,F,s_0)$ 
we construct a POMDP $\game = (\states',\act',\trans', \obs, \obsmap,s'_0)$ 
with a Boolean reward function $\reward$ such that there exists a word 
$w \in \act^*$  accepted with probability strictly greater than $\frac{1}{2}$ 
in $\automaton$ iff there exists a finite-memory almost-sure winning 
strategy in $\game$ for the objective $\limavghalf$.
Intuitively, the construction of the POMDP $\game$ is as follows: for every 
state $s \in \states$ of $\automaton$ we construct a pair of states $(s,1)$ and 
$(s,0)$ in $\states'$ with the property 
that $(s,0)$ can only be reached with a new action~$\$$ (not in $\act$) 
played in state $(s,1)$. The transition function $\trans'$ from the state 
$(s,0)$ mimics the transition function $\trans$, i.e., 
$\trans'((s,0),a)((s',1)) = \trans(s,a)(s')$.
The reward $\reward$ of $(s,1)$ (resp. $(s,0)$) is $1$ (resp. $0$), ensuring 
the average of the pair to be $\frac{1}{2}$. 
We add a new available action $\#$ that when played in a final state reaches a 
state $\good \in \states'$ with reward $1$, and when played in a non-final state
reaches a state $\bad \in \states'$ with reward $0$, 
and for states $\good$ and $\bad$ given action $\#$ the next state is the 
initial state.
An illustration of the construction on an example is depicted on Figure~\ref{fig:incompatible}.
Whenever an action is played in a state where it is not available, the POMDP reaches a loosing absorbing state, 
i.e., an absorbing state with reward $0$, and for brevity we omit transitions
to the loosing absorbing state.
The formal construction of the POMDP $\game$ is as follows:
\begin{itemize}
\item $\states' = (\states \times \{0,1\}) \cup \{\good,\bad\}$,

\item $s'_0 = (s_0,1)$,

\item $\act' = \act \cup \{\#,\$\}$,

\item The actions $a\in \act\cup \{\#\}$ in states $(s,1)$ (for $s \in S$) 
lead to the loosing absorbing state; the action $\$$ in states $(s,0)$ (for 
$s \in S$) leads to the loosing absorbing states; and the actions 
$a\in \act\cup \{\$\}$ in states $\good$ and $\bad$ lead 
to the loosing absorbing state.
The other transitions are as follows:
For all $s \in \states$: 
(i)~$\trans'((s,1),\$)((s,0)) = 1$, 
(ii)~for all $a \in \act$ we have $\trans'((s,0),a)((s',1)) = \trans(s,a)(s')$,
and 
(iii)~for action $\#$ we have 
	 $$
	\trans'((s,0),\#)(\good) = 
	\begin{cases}
	1 & \text{ if $s \in F$} \\
	0 & \text{ otherwise}
	\end{cases} \hspace{2em} 
	\trans'((s,0),\#)(\bad) = 
	\begin{cases}
	1 & \text{ if $s \not \in F$} \\
	0 & \text{ otherwise}
	\end{cases}$$
	$$\trans'(\good,\#)(s'_0) = 1 \hspace{2em} \trans'(\bad,\#)(s'_0) = 1,$$
	\item there is a single observation $\obs = \{o\}$, and all the states $s \in \states'$ have $\obsmap(s)=o$.
	\end{itemize}
We define the Boolean reward function $\reward$ only as a function of the 
state, i.e., $\reward:\states' \rightarrow \{0,1\}$ and show the 
undecidability even for this special case of reward functions. 
For all $s \in \states$ the reward is $ \reward((s,0)) = 0$, and similarly 
$\reward((s,1))=1$, and the remaining two states have rewards 
$\reward(\good) =1$ and $\reward(\bad) = 0$.
Note that though the rewards are assigned as function of states,
the rewards appear on the transitions.
We now establish the correctness of the reduction.

	\begin{figure}[ht]
			\begin{center}
			\begin{tikzpicture}[>=latex]
				\tikzstyle{every node}=[font=\small]

				\node[Player1,initial,initial text=]  (s0) {$s_0$};
				\node[Final, right of=s0,xshift=10]  (a) {$s$};
				\draw[->]{
					(s0) edge[] node[auto] {a} (a)
					(a) edge[loop] node[above] {b} (a)
				};

				\node[Player1,initial,initial text=,right of=a, xshift=50] (s01) {$s_0,1$};
				\node[Player1,right of=s01,xshift=30] (s00) {$s_0,0$};
				\node[Player1,right of=s00,xshift=30] (s11) {$s,1$};
				\node[Player1,right of=s11,xshift=30] (s10) {$s,0$};
				\node[Player1,inner sep=0.01cm,above of=s00,yshift=20] (good) {$\good$};
				\node[Player1,inner sep=0.1cm,below of=s00,yshift=-20] (bad) {$\bad$};

				\draw[->]{
					(s01) edge[] node[auto] {$\$$} (s00)
					(s00) edge[] node[auto] {$a$} (s11)
					(s11) edge[bend left] node[auto] {$\$$} (s10)
					(s10) edge[bend left] node[auto] {b} (s11)
					(s10) edge[bend right] node[right] {$\#$} (good)
					(good) edge[bend right] node[left] {$\#$} (s01)
					(s00) edge[] node[right] {$\#$} (bad)
					(bad) edge[bend left] node[left] {$\#$} (s01)
				};
			\end{tikzpicture}
					\caption{Transformation of the PFA $\automaton$ to a POMDP $\game$}
			\label{fig:incompatible}
			\end{center}
	\end{figure}
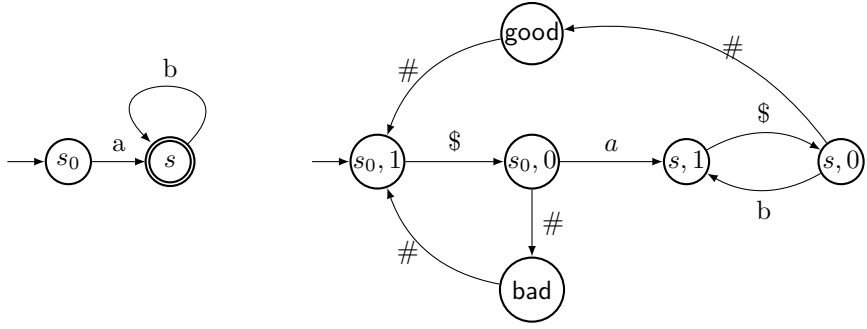

\begin{lemma}\label{lem:finquant_1}
If there exists a word $w\in\act^*$ accepted with probability strictly greater 
than $\frac{1}{2}$ in $\automaton$, then there exists a pure finite-memory 
almost-sure winning strategy in the POMDP $\game$ for the objective 
$\limavghalf$.
\end{lemma}
\begin{proof}
Let $w \in \act^*$ be a word accepted in $\automaton$ with probability  
$\mu > \frac{1}{2}$ and let the length of the word be $|w| = n$. 
We construct a pure finite-memory almost-sure winning strategy for the 
objective $\limavghalf$ objective in the POMDP $\game$ as follows: 
We denote by $w[i]$ the $i^{th}$ action in the word $w$. The finite-memory strategy we construct is specified as an ultimately periodic 
word  $( \$ \: w[1] \: \$ \: w[2] \ldots \: \$ w[n]  \: \# \: \#)^{\omega}$.
Observe that by the construction of the POMDP $\game$, the sequence of rewards (that appear on 
the transitions) is $(10)^n$ followed by (i)~$1$ with probability $\mu$ (when 
$F$ is reached), and (ii)~$0$ otherwise; and the whole sequence is repeated 
ad infinitum.
Also observe that once the pure finite-memory strategy is fixed we obtain a 
Markov chain with a single recurrent class since the starting state belongs to 
the recurrent class and all states reachable from the starting state 
form the recurrent class.
We first establish the almost-sure convergence of the sequence of partial 
averages of the periodic blocks, and then of the sequences inside the 
periodic blocks as well.

\smallskip\noindent{\em Almost-sure convergence of periodic blocks.}
Let $r_1, r_2, r_3, \ldots $ be the infinite sequence of rewards and 
$s_j  = \frac{1}{j} \cdot \sum_{i=1}^{j}r_i$. 
The infinite sequence of rewards can be partitioned into blocks of length 
$2\cdot n+1$, intuitively corresponding to the transitions of a single run on the word 
$(\$ \: w[1] \: \$ \: w[2] \ldots \: \$ w[n] \: \# \: \#)$. 
We define a random variable $X_i$ denoting average of rewards of the $i^{th}$ 
block in the sequence, i.e., with probability $\mu$ for all $i$ the value of 
$X_i$ is $\frac{n+1}{2\cdot n+1}$ and with probability $1-\mu$ the value is 
$\frac{n}{2\cdot n+1}$. 
The expected value of $X_i$ is therefore equal to $\Exp(X_i)=\frac{\mu+n}{2\cdot n+1}$, 
and as we have that $\mu > \frac{1}{2}$ it follows that $\Exp(X_i) > \frac{1}{2}$.
The fact that we have a single recurrent class and after the $\#\#$ the 
initial state is reached implies that the random variable sequence 
$(X_i)_{i\geq 0}$ is an infinite sequence of i.i.d's. 
By the Strong Law of Large Numbers (SLLN)~\cite[Theorem~7.1, page~56]{Durrett} 
we have that 
$$\prb\left(\lim_{j \rightarrow \infty} \frac{1}{j}(X_1 + X_2 + \ldots + X_j) = \frac{\mu+n}{2\cdot n+1}\right) = 1$$

\smallskip\noindent{\em Almost-sure convergence inside the periodic blocks.}
It follows that with probability $1$ the $\limavg$ of the partial averages on blocks is strictly greater than $\frac{1}{2}$. 
As the blocks are of fixed length $2\cdot n+1$, if we look at the sequence of 
averages at every $2\cdot n+1$ step, i.e., the sequence $s_{j\cdot (2n+1)}$ 
for $j > 0 $, we have that this sequence converges with probability $1$ 
to a value strictly greater than $\frac{1}{2}$. 
It remains to show that all $s_i$ converge to that value. 
As the elements of the subsequence converging with probability $1$ are always separated by exactly $2\cdot n +1$ elements (i.e.,
constant number of elements) and due to the definition of $\limavg$ the deviation introduced by these $2\cdot n +1$ 
elements ultimately gets smaller than any $\epsilon > 0$ as the length of the path increases 
(the average is computed from the whole sequence so far, and deviation caused
by a fixed length is negligible as the length increases). 
Therefore the whole sequence $(s_i)_{i >0}$ converges to a value strictly greater than $\frac{1}{2}$ with probability~1.
It follows that the strategy ensures $\limavghalf$ with probability~1.
\end{proof}

In the next two lemmas we first show the other direction for pure finite-memory strategies and then extend the result to the class of randomized infinite-memory
strategies.

\begin{lemma}
If there exists a pure finite-memory almost-sure winning strategy in the POMDP 
$\game$ for the objective $\limavghalf$, then there exists a word $w \in \act^*$
accepted with probability strictly greater than $\frac{1}{2}$ in $\automaton$ .
\end{lemma}
\begin{proof}
Assume there exists a pure finite-memory almost-sure winning strategy $\straa$ for the objective $\limavghalf$. 
Observe that as there is only a single observation in the POMDP $\game$ the strategy $\straa$ can be viewed as 
an ultimately periodic infinite word of the form $u\cdot v^{\omega}$, where 
$u,v$ are finite words from $\act'$. 
Note that $v$ must contain the subsequence $\# \#$, as otherwise the $\limavg$ 
would be only $\frac{1}{2}$. 
Similarly, before every letter $a \in \act$ in the words $u,v$, the strategy 
must necessarily play the $\$$ action, as otherwise the loosing absorbing 
state is reached.

In the first step we align the $\#\#$ symbols in $v$. Let us partition the 
word $v$ into two parts $v = y \cdot x$ such that $y$ is the shortest prefix 
ending with $\#\#$. Then the ultimately periodic word $u \cdot y \cdot 
(x\cdot y)^{\omega} = u \cdot v^\omega$ is also a strategy ensuring 
almost-surely $\limavghalf$.
Due to the previous step we consider $u'=u \cdot y$ and $v'= x\cdot y$, and thus have that 
$v'$ is of the form:
	 $$ \$  w_1[1]  \$  w_1[2] \ldots \$ w_1[n_1]  \#\# \$ w_2[1] \$ w_2[2] \ldots \$ w_2[n_2] \# \# \ldots \$ w_m[1]  \$  w_m[2] \ldots \$ w_m[n_m] \# \#$$
	 We extract the set of words $W =  \{w_1,w_2, \ldots, w_m\}$ from $v'$. Assume towards contradiction that all the words in the set $W$ are accepted in the PFA $\automaton$ with probability at most $\frac{1}{2}$. 
As in Lemma~\ref{lem:finquant_1} we define a random variable $X_i$ denoting the average of rewards after reading $v'$. 
It follows that the expected value of $\Exp(X_i) \leq \frac{1}{2}$ for all $i \geq 0$. By using SLLN we obtain that almost-surely the $\limavg$ of $u\cdot v^{\omega}$ is $\Exp(X_i)$, and hence it 
is not possible as $u\cdot v^{\omega}$  is an almost-sure winning strategy for 
the objective $\limavghalf$.
We reach a contradiction to the assumption that all the words in $W$ are accepted with probability at most $\frac{1}{2}$ in $\automaton$. 
Therefore there exists a word $w \in W$ that is accepted in ${\automaton}$ 
with probability strictly greater than $\frac{1}{2}$, which concludes the proof.
\end{proof}

	To complete the reduction we show in the following lemma that pure strategies are sufficient for the POMDPs constructed in our reduction.

\begin{lemma}
Given the POMDP $\game$ of our reduction, if there is a randomized (possibly infinite-memory)
almost-sure winning strategy for the objective $\limavghalf$,
then there exists a pure finite-memory almost-sure winning strategy $\straa'$ for the objective $\limavghalf$.
\end{lemma}
\begin{proof}
Let $\straa$ be a randomized (possibly infinite-memory) almost-sure winning strategy for the objective $\limavghalf$.
As there is a single observation in the POMDP $\game$ constructed in the reduction, the strategy does not receive any useful
feedback from the play, i.e., the memory update function $\straa_u$ always receives as one of the parameters the unique observation.
Note that the strategy needs to play the pair $\#\#$ of actions infinitely often with probability $1$, i.e., with probability $1$
the resolving of the probabilities in the strategy $\straa$ leads to an infinite word $\rho  = w_1 \#\# w_2 \#\# \ldots$,
as otherwise the limit-average payoff is at most $\frac{1}{2}$ with positive probability.
From each such run $\rho$ we extract the finite words $w_1,w_2, \ldots$ that occurs in $\rho$,
and then consider the union of all such words as $W$.
We consider two options:
\begin{enumerate}
 \item If there exists a word $v$ in $W$ such that the expected average of rewards after playing this word is strictly greater than $\frac{1}{2}$, then the pure strategy $v^{\omega}$ is also a pure finite-memory almost-sure winning strategy for the objective $\limavghalf$.
        
\item Assume towards contradiction that all the words in $W$ have the expected reward at most $\frac{1}{2}$.
Then with probability $1$ resolving the probabilities in the strategy $\straa$ leads to an infinite word $\wb{w} = w_1 \#\# w_2 \#\# \ldots$, 
where each word $w_i$ belongs to  $W$, that is played on the POMDP $\game$.
Let us define a random variable $X_i$ denoting the average between $i$ and $(i+1)$-th occurrence of $\#\#$.
The expected average $\Exp(X_i)$ is at most $\frac{1}{2}$ for all $i$.
Therefore the expected $\limavg$ of the sequence $\wh{w}$ is at most:

$$ \Exp(\liminf_{n\rightarrow \infty} \frac{1}{n} \sum\limits_{i=0}^{n} X_i).$$

Since $X_i$'s are non-negative measurable functions, by Fatou's lemma~\cite[Theorem~3.5, page 16]{Durrett} 
that shows the integral of limit inferior of a sequence of non-negative 
measurable functions is at most the limit inferior of the integrals of these 
functions, we have the following inequality:

$$ \Exp(\liminf_{n\rightarrow \infty} \frac{1}{n} \sum\limits_{i=0}^{n} X_i) \leq  \liminf_{n \rightarrow \infty} \Exp(\frac{1}{n} \sum\limits_{i=0}^{n} X_i) \leq \frac{1}{2}.$$

Note that since the strategy $\straa$ is almost-sure winning  for the objective $\limavghalf$, 
then the expected value of rewards must be strictly greater than $\frac{1}{2}$. 
Thus we arrive at a contradiction.
Hence there must exist a word in $W$ that that has an expected payoff 
strictly greater than $\frac{1}{2}$ in $\game$.
\end{enumerate}
This concludes the proof.
\end{proof}

\begin{theorem}
The problem whether there exists a finite (or infinite-memory) 
almost-sure winning strategy in a POMDP for the objective $\limavghalf$ is 
undecidable.
\end{theorem}

\section{Infinite-memory strategies with Qualitative Constraint}
In this section we show that the problem of deciding the existence of 
infinite-memory almost-sure winning strategies in POMDPs with $\limavgone$ 
objectives is undecidable. 
We prove this fact by a reduction from the \emph{value~1 problem} in 
PFA, which is undecidable~\cite{GO10}.
The value~$1$ problem given a PFA $\automaton$ asks whether for every 
$\epsilon > 0$ there exists a finite word $w$ such that the word is accepted
in ${\automaton}$ with probability at least $1-\epsilon$ (i.e., the limit of
the acceptance probabilities is~1).

\smallskip\noindent{\bf Reduction.}
Given a PFA $\automaton = (\states, \act, \trans,F,s_0)$, 
we construct a POMDP $\game' = (\states',\act',\trans', \obs', \obsmap',s'_0)$ 
with a reward function $\reward'$, such that $\automaton$ satisfies the value 
$1$ problem iff there exists an infinite-memory almost-sure winning strategy in 
$\game'$ for the objective $\limavgone$. 
Intuitively, the construction adds two additional states  $\good$ and $\bad$. 
We add an edge from every state of the PFA under a new action $\$$, this edge 
leads to the state $\good$ when played in a final state, and to the state 
$\bad$ otherwise. 
In the states $\good$ and $\bad$ we add self-loops under a new action $\#$.
The action $\$$ in the states $\good$ or $\bad$ leads back to the initial 
state. 
An example of the construction is illustrated with Figure~\ref{fig:infconst}. 
All the states belong to a single observation, and we will use Boolean reward
function on states.
The reward for all states except the newly added state $\good$ is~$0$, 
and the reward for the state $\good$ is~$1$.
The formal construction is as follows:

\begin{itemize}
\item $\states' = \states \cup \{\good,\bad \}$,
\item $s'_0 = s_0$,
\item $\act' = \act \cup \{\#,\$\}$,
\item For all $s,s' \in \states$ and $a \in \act$ we have 
$\trans'(s,a)(s') = \trans(s,a)(s')$,
$$\trans'(s,\$)(\good) = 
\begin{cases}
1 & \text{ if $s \in F$} \\
0 & \text{ otherwise}
\end{cases} \hspace{2em} 
\trans'(s,\$)(\bad) = 
\begin{cases}
1 & \text{ if $s \not \in F$} \\
0 & \text{ otherwise}
\end{cases}$$
$$\trans'(\good,\$)(s_0) =  \trans'(\bad,\$)(s_0) = 1,$$
$$\trans'(\good,\#)(\good) = \trans'(\bad,\#)(\bad) = 1$$
\item there is a single observation $\obs = \{o\}$, and all the states 
$s \in \states'$ have $\obsmap(s)=o$.
\end{itemize}

	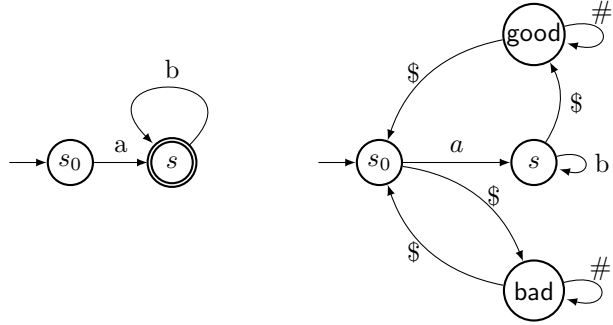
\begin{figure}[ht]
			\begin{center}
			\begin{tikzpicture}[>=latex]
				\tikzstyle{every node}=[font=\small]

				\node[Player1,initial,initial text=]  (s0) {$s_0$};
				\node[Final, right of=s0,xshift=10]  (a) {$s$};
				\draw[->]{
					(s0) edge[] node[auto] {a} (a)
					(a) edge[loop] node[above] {b} (a)
				};

				\node[Player1,initial,initial text=,right of=a, xshift=50] (s01) {$s_0$};
				\node[Player1,right of=s01,xshift=30] (s11) {$s$};
				\node[Player1,inner sep=0.01cm,above of=s00,yshift=20] (good) {$\good$};
				\node[Player1,inner sep=0.08cm,below of=s00,yshift=-20] (bad) {$\bad$};

				\draw[->]{
					(s01) edge[] node[auto] {$a$} (s11)
					(s11) edge[loop right] node[auto] {b} (s11)
					(s11) edge[bend right] node[right] {$\$$} (good)
					(s01) edge[bend left] node[right] {$\$$} (bad)
					(good) edge[loop right] node[above] {$\#$} (good)
					(bad) edge[loop right] node[above] {$\#$} (bad)
					(good) edge[bend right] node[left] {$\$$} (s01)
					(bad) edge[bend left] node[left] {$\$$} (s01)
				};
			\end{tikzpicture}
					\caption{Transformation of the PFA $\automaton$ to a POMDP $\game$}
			\label{fig:infconst}
			\end{center}
	\end{figure}

When an action is played in a state without an outgoing edge for the action, 
the loosing absorbing state is reached; i.e., for action $\#$ in states in $S$
and actions $a \in \act$ for states $\good$ and $\bad$, the next state is
the loosing absorbing state.
The Boolean reward function $\reward':\states' \rightarrow \{0,1\}$ assigns all states  $s \in \states \cup \{\bad\}$ the reward $ \reward'(s) = 0$, and $\reward'(\good) = 1$.

\begin{lemma}
If the PFA $\automaton$ satisfies the value $1$ problem, then there exists an infinite-memory almost-sure winning strategy ensuring the $\limavgone$ objective.
\end{lemma}

\begin{proof}
We construct an almost-sure winning strategy $\straa$ that we describe as 
an infinite word. 
As $\automaton$ satisfies the value $1$  problem, there exists a sequence of 
finite words $(w_i)_{i \geq 1}$, such that each $w_i$ is accepted in 
$\automaton$ with probability at least $1-\frac{1}{2^{i+1}}$. 
We construct an infinite word $w_1 \cdot \$ \cdot \#^{n_1} \cdot w_2\cdot  \$ \cdot \#^{n_2} \cdots$, 
where each $n_i \in \mathbb{N}$ is a natural number that satisfies the following 
condition: let $k_i=|w_{i+1}\cdot \$| + \sum_{j=1}^{i} (|w_j \cdot \$|+n_j)$ be the length 
of the word sequence before $\#^{n_{i+1}}$, then we must have 
$\frac{n_i}{k_i} \geq 1-\frac{1}{i}$. 
In other words, the length $n_i$ is long enough such that even if the whole
sequence of length $\sum_{j=1}^{i} (|w_j \cdot \$|+n_j)$ before $n_i$ and the 
sequence of length $|w_{i+1} \cdot \$| $ is zero's, the sequence $n_i$ one's ensures
the average is at least $1-\frac{1}{i}$. 
Intuitively the condition ensures that even if the rewards are always $0$ for 
the prefix up to $w_i\cdot \$$ of a run, even then a single visit to the $\good$ 
state after $w_i \cdot \$$ can ensure the average to be greater $1-\frac{1}{i}$ from the 
end of $\#^{n_i}$ up to the point $w_{i+1} \cdot \$$ ends.
We first argue that if the state $\bad$ is visited finitely often with 
probability~1, then with probability~1 the objective $\limavgone$ is satisfied.

\smallskip\noindent{\em Almost-sure winning if $\bad$ only finitely often.}
Observe that if the state $\bad$ appears only finitely often, then from some 
$j \geq 0$, for all $\ell \geq j$, 
the state visited after $w_1 \cdot \$ \cdot \#^{n_1} \cdot w_2\cdot  \$ \cdot \#^{n_2} \cdots w_{\ell} \cdot \$$
is the state $\good$, and then the sequence $\#^{n_{\ell}}$ ensures that the 
payoff from the end of $\#^{n_{\ell}}$ up to the end of $w_{\ell+1} \cdot \$$ is at 
least $1-\frac{1}{\ell}$.
If the state $\good$ is visited, then the sequence of $\#$ gives reward~1, and
thus after a visit to the state $\good$, the average only increases in the 
sequence of $\#$'s. 
Hence it follows that lim-inf average of the rewards is at least 
$1-\frac{1}{i}$, for all $i\geq 0$ (i.e., for every $i\geq 0$, there exists a 
point in the path such that the average of the rewards never falls below 
$1-\frac{1}{i}$).
Since this holds for all $i \geq 0$, it follows that $\limavgone$ is ensured 
with probability~1, (provided with probability~1 the state $\bad$ appears only 
finitely often).

\smallskip\noindent{\em The state $\bad$ only finitely often.}
We now need to show that the state $\bad$ is visited infinitely often with 
probability~0 (i.e., only finitely often with probability~1). 
We first upper bound the probability $u_{k+1}$ to visit the state $\bad$ at 
least $k+1$ times, given $k$ visits to state $\bad$.
The probability $u_{k+1}$ is at most 
$ \frac{1}{2^{k+1}}( 1 + \frac{1}{2} + \frac{1}{4} + \cdots) $.
The above bound for $u_{k+1}$ is obtained as follows: following the visit to 
$\bad$ for $k$ times, the words $w_j$, for $j\geq k$ are played;
and hence the probability to reach $\bad$ decreases by $\frac{1}{2}$ every 
time the next word is played; and after $k$ visits the probability is always
smaller than $\frac{1}{2^{k+1}}$.
Hence the probability to visit $\bad$ at least $k+1$ times, given $k$ visits, 
is at most the sum above, which is $\frac{1}{2^{k}}$.
Let $\cale_k$ denote the event that $\bad$ is visited at least $k+1$ times 
given $k$ visits to $\bad$.
Then we have 
$\sum_{k\geq 0} \prb(\cale_k) \leq \sum_{k\geq 1} \frac{1}{2^{k}} < \infty$.
By Borel-Cantelli lemma~\cite[Theorem~6.1, page 47]{Durrett} we know that if 
the sum of probabilities is finite (i.e., $\sum_{k\geq 0} \prb(\cale_k) < \infty$), then the probability 
that infinitely many of them occur is~0 
(i.e., $\prb(\lim \sup_{k\to\infty}\cale_k)=0$).
Hence the probability that $\bad$ is visited infinitely often is~0, i.e., with 
probability~1 $\bad$ is visited finitely often.

It follows that the strategy $\straa$ that plays the infinite word 
$w_1 \cdot \$ \cdot \#^{n_1} \cdot w_2\cdot  \$ \cdot \#^{n_2} \cdots$
is an almost-sure winning strategy for the objective $\limavgone$.
\end{proof}

\begin{lemma}
If there exists an infinite-memory almost-sure winning strategy for the objective $\limavgone$, then the PFA $\automaton$ 
satisfies the value $1$ problem.
\end{lemma}
\begin{proof}
We prove the converse. Consider that the PFA $\automaton$ does not satisfy the 
value $1$ problem, i.e., there exists a constant $c > 0$ such that for all 
$w \in \act^{*}$ we have that the probability that $w$ is accepted in 
${\automaton}$ is at most $1-c<1$.
We will show that there is no almost-sure winning strategy.
Assume towards contradiction that there exists an infinite-memory almost-sure 
winning strategy $\straa$ in the POMDP $\game'$. 
In POMDPs infinite-memory pure strategies are as powerful as infinite-memory
randomized strategies\footnote{Since in POMDPs there is only the controller 
making choices, a randomized strategy can be viewed as a distribution over 
pure strategies, and if there is a randomized strategy to achieve an objective, 
then there must be a pure one. This is the key intuition that in POMDPs randomization is not more powerful; for a formal proof, see~\cite{CDGH10}.}.
Therefore we may assume that the almost-sure winning strategy is given as an infinite word $\wb{w}$. 
Note that the infinite word $\wb{w}$ must necessarily contain infinitely many $\$$ and can be written as 
$\wb{w}= w_1 \cdot \$ \cdot \#^{n_1} \cdot \$ \cdot w_2 \cdot \$ \cdot \#^{n_2} \cdot \$ \cdots$. 
Moreover, there must be infinitely many $i>0$ such that the number $n_i$ is  positive (since only such 
segments yield reward~1).

Consider the Markov chain $\game' \restr_{\straa}$ and the rewards on the states of the chain. By the definition of the reward function $\reward'$ all the rewards that correspond to words $w_i$ for $i>0$ are equal to $0$. Rewards corresponding to the word segment $\#^{n_i}$ for $i>0$ are with probability at most $1-c$ equal to $1$ and $0$ otherwise. Let us for the moment remove the $0$ rewards corresponding to words $w_i$ for all $i>0$ (removing the reward~0 segments 
only increases the limit-average payoff), and consider only rewards corresponding to the segments $\#^{n_i}$ for all $i>0$, we will refer to this infinite sequence of rewards as $\wh{w}$. 
Let $X_i$ denote the random variable corresponding to the value of the $i^{th}$ reward in the sequence $\wh{w}$. 
Then we have that $X_i = 1$ with probability at most $1-c$ and $0$ otherwise. The expected $\limavg$ of the sequence $\wh{w}$ is then at most:
$$\Exp(\liminf\limits_{n\rightarrow \infty} \frac{1}{n} \sum_{i=0}^{n} X_i).$$
Since $X_i$'s are non-negative measurable function, by Fatou's lemma~\cite[Theorem~3.5, page 16]{Durrett} 
that shows the integral of limit inferior of a sequence of non-negative 
measurable functions is at most the limit inferior of the integrals of these 
functions, we have the following inequality:
$$\Exp(\liminf\limits_{n\rightarrow \infty} \frac{1}{n} \sum_{i=0}^{n} X_i) 
\leq \liminf\limits_{n \rightarrow \infty} \Exp(\frac{1}{n} \sum_{i=0}^{n} X_i) \leq 1-c.$$

If we put back the rewards $0$ from words $w_i$ for all $i>0$, then the 
expected value can only decrease.
It follows that $\Exp^{\straa}(\limavg) \leq 1-c$.
Note that if the strategy $\straa$ was almost-sure winning for the objective 
$\limavgone$ (i.e., $\prb^{\straa}(\limavgone)=1$), 
then the expectation of the $\limavg$ payoff would also be $1$ 
(i.e., $\Exp^{\straa}(\limavg) =1$). 
Therefore we have reached a contradiction to the fact that the strategy 
$\straa$ is almost-sure winning, and the result follows.
\end{proof}

\begin{theorem}
The problem whether there exists an infinite-memory almost-sure winning strategy in a POMDP with the objective $\limavgone$ is undecidable.
\end{theorem}

\section{Conclusion}
We studied POMDPs with limit-average objectives under probabilistic semantics. 
Since for general probabilistic semantics, the problems are undecidable even 
for PFA, we focus on the very important special case of almost-sure winning. 
For almost-sure winning with qualitative constraint, we show that belief-based 
strategies are not sufficient for finite-memory strategies, and establish 
EXPTIME-complete complexity for the existence of finite-memory strategies.
Given our decidability result, the next natural questions are whether 
the result can be extended to infinite-memory strategies, or to quantitative 
path constraint. 
We show that both these problems are undecidable, and thus establish the 
precise decidability frontier with optimal complexity.
Also observe that contrary to other classical results for POMDPs where both the
finite-memory and infinite-memory problems are undecidable, for almost-sure 
winning with qualitative constraint, we show that the finite-memory problem is 
decidable (EXPTIME-complete), but the infinite-memory problem is undecidable.

\bibliographystyle{plain}
\bibliography{diss}

\end{document}